\documentclass[times, review, 10pt]{elsarticle}
\usepackage{amssymb}
\usepackage{graphicx}
\usepackage{amsmath}
\usepackage{longtable}
\usepackage{a4wide}
\usepackage{mathrsfs}
\usepackage{subcaption}
\usepackage{mathtools}
\usepackage{pifont}
\usepackage{algorithmic}
\usepackage{algorithm}
\usepackage{xcolor}
\usepackage{color}
\usepackage{lineno}
\usepackage{makecell} 
\usepackage{pdflscape}
\usepackage{adjustbox}
\usepackage[utf8]{inputenc}
\usepackage{tabularx}
\usepackage{setspace}
\usepackage{blindtext}
\usepackage{multirow}
\usepackage{notoccite} 
\usepackage{lscape} 
\usepackage{caption} 
\usepackage{mwe}
\usepackage{tikz}
\usepackage{siunitx}
\usepackage{mathrsfs}
\usetikzlibrary{shapes,arrows}
\usepackage{xcolor}
\usepackage{booktabs}
\usepackage{hyperref}
\usepackage{amsthm}
\newtheorem{theorem}{Theorem}[section]

\theoremstyle{definition}
\newtheorem{definition}{Definition}
\newtheorem{lemma}[theorem]{Lemma}

\newcommand{\RNum}[1]{\lowercase\expandafter{\romannumeral #1\relax}}
\newcommand{\RNumU}[1]{\uppercase\expandafter{\romannumeral #1\relax}}
\usepackage{natbib}
\journal{Elsevier}

\begin{document}
\date{}
\begin{frontmatter}
\title{Enhancing Multiview Synergy: Robust Learning by Exploiting the Wave Loss Function with Consensus and Complementarity Principles}
\author[inst1]{A. Quadir}
\ead{mscphd2207141002@iiti.ac.in}
\affiliation[inst1]{organization={Department of Mathematics},
            addressline={Indian Institute of Technology Indore}, 
            city={Simrol, Indore},
            postcode={453552}, 
            country={India}}
\author[inst1]{Mushir Akhtar}
\ead{phd2101241004@iiti.ac.in}
\author[inst1]{M. Tanveer\texorpdfstring{\corref}{corref}{Correspondingauthor}}
 \cortext[Correspondingauthor]{Corresponding author}
\ead{mtanveer@iiti.ac.in}

\begin{abstract}
Multiview learning (MvL) is an advancing domain in machine learning, leveraging multiple data perspectives to enhance model performance through view-consistency and view-discrepancy. Despite numerous successful multiview-based support vector machine (SVM) models, existing frameworks predominantly focus on the consensus principle, often overlooking the complementarity principle. Furthermore, they exhibit limited robustness against noisy, error-prone, and view-inconsistent samples, prevalent in multiview datasets. To tackle the aforementioned limitations, this paper introduces Wave-MvSVM, a novel multiview support vector machine framework leveraging the wave loss (W-loss) function, specifically designed to harness both consensus and complementarity principles. Unlike traditional approaches that often overlook the complementary information among different views, the proposed Wave-MvSVM ensures a more comprehensive and resilient learning process by integrating both principles effectively. The W-loss function, characterized by its smoothness, asymmetry, and bounded nature, is particularly effective in mitigating the adverse effects of noisy and outlier data, thereby enhancing model stability. Theoretically, the W-loss function also exhibits a crucial classification-calibrated property, further boosting its effectiveness. The proposed Wave-MvSVM employs a between-view co-regularization term to enforce view consistency and utilizes an adaptive combination weight strategy to maximize the discriminative power of each view, thus fully exploiting both consensus and complementarity principles. The optimization problem is efficiently solved using a combination of gradient descent (GD) and the alternating direction method of multipliers (ADMM), ensuring reliable convergence to optimal solutions. Theoretical analyses, grounded in Rademacher complexity, validate the generalization capabilities of the proposed Wave-MvSVM model. Extensive empirical evaluations across diverse datasets demonstrate the superior performance of Wave-MvSVM in comparison to existing benchmark models, highlighting its potential as a robust and efficient solution for MvL challenges.
\end{abstract}
\end{frontmatter}
\section{Introduction}
The proliferation of data across various domains has underscored the critical need for advanced methodologies that can effectively handle and integrate multiple sources of information. Multiview learning (MvL), a paradigm that seeks to leverage the complementary information available from different views or data modalities, has emerged as a potent approach to enhance learning performance \cite{tian2022multi, tang2024incomplete}. This approach is particularly relevant in fields such as computer vision \cite{yan2022multiview, luo2020attention}, bioinformatics \cite{fu2022mvgcn, serra2024multiview}, and natural language processing \cite{sadr2020multi, long2022multi}, where data often comes from diverse and heterogeneous sources. Conventional single-view learning techniques frequently fail to capture the intricate relationships and interactions inherent in multiview data, thereby limiting their effectiveness \cite{zhao2017multi}. Consequently, there has been a significant push in the research community towards developing MvL frameworks that can seamlessly integrate and exploit the synergies among multiple views, enhancing overall learning performance.
\par
MvL approaches can be broadly categorized into two principles: consensus and complementarity \cite{zhu2022fast, zhao2022robust}. The consensus principle aims to maximize the agreement among different views, ensuring that each view contributes to a unified learning objective. On the other hand, the complementarity principle emphasizes leveraging the unique information present in each view to enrich the overall model representation. Adhering to these two principles, existing MvL studies integrate multiview representations using three primary schemes: early fusion, late fusion, and a combination of both fusion types. In early fusion techniques, all views are amalgamated into a single large feature vector prior to any further processing \cite{yu2011optimized, zheng2019feature}. The model is then trained on this concatenated view. However, this approach often disregards the correlations and interactions among views, potentially leading to overfitting and dimensionality issues \cite{li2018review}. Late fusion techniques operate at the decision level by training on each view separately and subsequently combining the predictions in an ensemble manner \cite{gupta2020novel, zhang2021multi}. While this method offers some flexibility, it fails to fully exploit the consensus and complementary information inherent in different views. The combined fusion strategy aims to balance view-consistency and view diversity, ultimately generating the final predictors collectively. Given that views inherently describe the same objects in different feature spaces, this hybrid approach often results in superior performance across most multiview classification methods by leveraging both view-specific and shared information effectively \cite{meng2020multiview, van2020stacked, xie2020general}.
\par
Support vector machines (SVMs) \cite{cortes1995support, pisner2020support} have been a cornerstone in the realm of machine learning due to their robustness and efficacy in interpretability. SVMs are formulated within a cohesive framework that integrates regularization terms with loss functions \cite{akhtar2023roboss, abdul2024granular}. However, classical SVMs are typically designed for single-view scenarios, thereby underutilizing the multifaceted nature of data \cite{huang2016multi}. In single-view settings, SVMs optimize a hyperplane that maximizes the margin between different classes based on a single set of features \cite{kumari2024diagnosis}. This limitation has prompted researchers to explore the fusion of multiview data into the SVM framework, leading to the development of numerous multiview SVM models.
\par
For instance, SVM-2K \cite{farquhar2005two} integrates kernel canonical correlation analysis (KCCA) \cite{zhuang2020technical} with SVM into a single optimization framework, enforcing view-consistency constraints. Multiview Laplacian SVMs \cite{sun2011multi} extend traditional SVMs by incorporating manifold and multiview regularization, thus bridging supervised and semi-supervised learning. Additionally, \citet{sun2010sparse} introduced a general multiview framework by fusing a $\varepsilon$-insensitive loss for sparse semi-supervised learning. The multiview L2-SVM \cite{huang2016multi} capitalizes on both coherence and disparity across views by enforcing consensus among multiple views. Another significant contribution is the multiview nonparallel SVM \cite{tang2018multi}, which embeds the nonparallel SVM into a multiview classification framework under the consensus constraint. Furthermore, \citet{sun2018multiview} developed a multiview generalized eigenvalue proximal SVM by incorporating a multiview co-regularization term to maximize consensus across diverse views.
In the realm of sentiment analysis, \citet{ye2021multi} devised a novel multiview ensemble learning method that fuses information from different features to enhance microblog sentiment classification. The aforementioned models typically adhere to either the consensus or complementarity principle, but rarely both in unison.
\par
In recent years, numerous SVM-based MvL methodologies have proficiently integrated both the principles of consensus and complementarity. Noteworthy instances encompass multiview privileged SVM \cite{tang2017multiview}, multiview SVM \cite{xie2019multi} that amalgamate consensus and complementary information, and multiview twin SVM \cite{xie2019multi} with similar integration paradigms. These models predominantly utilize conventional loss functions, such as hinge loss, which inadequately capture the intricate relationships among diverse views, leading to suboptimal performance, particularly in the presence of noisy or extraneous features. To surmount these challenges, frameworks incorporating the linear-exponential (LINEX) loss function \cite{tang2021multi} and the quadratic type squared error (QTSE) loss function \cite{hou2024mvqs} have been introduced. Among the aforementioned models, it is notable that \cite{tang2021multi} and \cite{hou2024mvqs} are capable of not only fully leveraging both consistency and complementarity information but also employing asymmetric LINEX loss and QTSE loss functions, respectively. These functions adeptly differentiate between samples prone to errors and those less likely to errors, positioned between and outside the central hyperplanes. Nonetheless, these models often exhibit instability when confronted with outliers. The unbounded nature of both LINEX and QTSE loss functions permits outlier entries with substantial errors to exert disproportionate influence on the final predictor, thereby significantly skewing the decision hyperplane from its optimal position. Consequently, the implementation of a more flexible and bounded loss function is imperative to mitigate these issues and advance the field of MvL.
\par
To advance the field of MvL, we propose the fusion of the wave loss (W-loss) function into the multiview SVM framework, seamlessly amalgamating both consensus and complementarity principles for enhanced performance and robustness. Recently, \citet{akhtar2024advancing} developed the W-loss, which offers several advantages over traditional loss functions. It possesses nice mathematical properties, including smoothness, asymmetry, and boundedness. The smoothness characteristic of the W-loss function facilitates the use of gradient-based optimization techniques by ensuring the existence of well-defined gradients, thereby enabling efficient and reliable optimization methods. The asymmetry characteristic allows for the assignment of diverse punishments to samples depending on their propensity for misclassification. Lastly, the boundedness of the W-loss imposes an upper limit on loss values, preventing excessive influence from outliers. This upper bound is particularly advantageous in the presence of outliers, as it prevents the loss from escalating uncontrollably due to extreme errors. Theoretically, the W-loss function also exhibits a crucial classification-calibrated property. Consequently, the model becomes robust to outliers, maintaining stability and performance even in the presence of anomalous data points. Further, to enhance view-consistency, we introduce a between-view co-regularization term in the objective function, applied to the predictors of two views. Additionally, we exploit complementarity information by adaptively adjusting the combination weight for each view, which emphasizes more important and discriminative views. Thus, by fusing the W-loss function into the multiview SVM framework while adhering to both consensus and complementarity principles, we propose a novel and robust multiview SVM framework named Wave-MvSVM. The optimization problem of Wave-MvSVM is addressed using the gradient descent (GD) algorithm in conjunction with the alternating direction method of multipliers (ADMM). Additionally, the generalization capability of Wave-MvSVM is rigorously analyzed through the lens of Rademacher complexity. The main highlights of this study are encapsulated as follows:
\begin{enumerate}
    \item We propose Wave-MvSVM, a novel multiview SVM framework that effectively amalgamates the wave loss (W-loss) function, adeptly exploiting both consensus and complementarity principles to significantly enhance performance and robustness in MvL.
    \item We utilized the W-loss function as the error metric, which not only flexibly differentiates error-prone samples across both classes but also fortifies the model's robustness against noisy and anomalous data points, maintaining stability even in challenging conditions.
    \item We employed a hybrid approach combining gradient descent (GD) and the alternating direction method of multipliers (ADMM) to solve the optimization problem for Wave-MvSVM, ensuring efficient and reliable convergence to optimal solutions.
    \item We rigorously analyzed the generalization capability of the proposed Wave-MvSVM framework using Rademacher complexity, thereby offering theoretical assurances for its performance across various scenarios.
    \item We conducted extensive numerical experiments to validate the effectiveness of the Wave-MvSVM framework, demonstrating its superior performance in comparison to benchmark models across diverse datasets.
\end{enumerate}
The remainder of this paper is organized as follows. Section \ref{Related Work} provides an overview of related work. Section \ref{Proposed Work} details the W-loss function and provides the mathematical formulation of the proposed Wave-MvSVM model. Section \ref{Optimizaiton for Wave-MvSVM} presents the optimization techniques employed. Section \ref{Theoretical Analysis} discusses the theoretical analysis of the proposed Wave-MvSVM model. Section \ref{Experiments and Results} discusses the experimental setup and results, providing a thorough comparison of the proposed Wave-MvSVM with baseline models. Finally, the conclusions and potential future research directions are given in Section \ref{Conclusion and Future Work}.
\section{Related Work}
\label{Related Work}
This section begins with establishing notations and then reviews the mathematical formulation along with the solution of the SVM-2K.
\subsection{Notations}
Consider the sample space denoted as $\mathscr{T}$, which is a product of two distinct feature views, $1$ and $2$, expressed as $\mathscr{T} = \mathscr{T}^{[1]} \times \mathscr{T}^{[2]}$, where $\mathscr{T}^{[1]} \subseteq \mathbb{R}^{m_1}$, $\mathscr{T}^{[2]} \subseteq \mathbb{R}^{m_2}$ and $\mathscr{Y} = \{-1, +1\}$ denotes the label space. Here $m_1$ and $m_2$ denote the number of features corresponding to view $1$ and view $2$, respectively. Suppose $\mathscr{H} = \{(x_i^{[1]}, x_i^{[2]}, y_i) | x_i^{[1]} \in \mathscr{T}^{[1]}, x_i^{[2]} \in \mathscr{T}^{[2]}, y_i \in \mathscr{Y}\}_{i=1}^n$ represent a two-view data set, where $n$ represents the number of samples.
\subsection{Two view learning: SVM-2K, Theory and Practice (SVM-2K)}
SVM-2K \cite{farquhar2005two} finds two distinct optimal hyperplanes: one associated with view $1$ and another with view $2$, and given as:
\begin{align}
    w^{{[1]}^T}\phi_1(x^{[1]}) + b^{[1]} =0,  \hspace{0.4cm} \text{and} \hspace{0.4cm}  w^{{[2]}^T}\phi_2(x^{[2]}) + b^{[2]} =0.
\end{align}
The optimization problem of SVM-2K can be written as follows:
\begin{align}
\label{eq:1}
\underset{ w_1, w_2, b_1, b_2}{\min}  \hspace{0.1cm}~&\frac{1}{2}\|w_1\|^2+\frac{1}{2}\|w_2\|^2+C_{1}\sum_{i=1}^{n} \zeta_{i}^{[1]} + C_{2}\sum_{i=1}^{n} \zeta_{i}^{[2]} + D\sum_{i=1}^{n} \eta_{i} \nonumber \\
 \text { s.t. }\hspace{0.1cm}  & \lvert <w_2, \phi(x_i^{[2]})> + b_2 - <w_1, \phi(x_i^{[1]})> - b_1 \rvert \leq \epsilon + \eta_i, \nonumber \\
 & y_i(<w^{[2]}, \phi(x_i^{[2]})> + b^{[2]}) \geq 1 - \zeta_i^{[2]}, \nonumber \\
 & y_i(<w^{[1]}, \phi(x_i^{[1]})> + b^{[1]}) \geq 1 - \zeta_i^{[1]}, \nonumber \\
 & \zeta_{i}^{[1]} \geq 0, \zeta_{i}^{[2]} \geq 0, \eta_{i} \geq 0, ~~ i=1, 2, \ldots n, 
\end{align}
where $C_{1}$, $C_2$, $D$ $(> 0)$ are penalty parameters, $\epsilon > 0$ is an insensitive parameter, $\zeta_i^{[1]}$, $\zeta_i^{[2]}$, $\eta_i$ are slack variables, and $\phi$ denotes the feature mapping function. The amalgamation of the two views through the  $\epsilon$-insensitive $l_1$-norm similarity constraint is incorporated as the initial constraint in the problem \eqref{eq:1}.
\par
The dual problem of \eqref{eq:1} is given by:
\begin{align}
\label{eq:2}
\underset{ q^{[1]}, q^{[2]}, \alpha_i^{[1]}, \alpha_i^{[2]}}{\max}  \hspace{0.05cm}~&-\frac{1}{2} \sum_{i, j=1}^n(q_i^{[1]}q_j^{[1]}k_1(x_i, x_j) + q_i^{[2]}q_j^{[2]}k_2(x_i, x_j)) + \sum_{i=1}^n(\alpha_i^{[1]}+ \alpha_i^{[2]}) \nonumber \\
 \text { s.t. }\hspace{0.1cm}  & q_i^{[1]} = \alpha_i^{[1]}y_i - \beta_i^+ + \beta_i^-, \nonumber \\
 & q_i^{[2]} = \alpha_i^{[2]}y_i + \beta_i^+ - \beta_i^-, \nonumber \\
 & \sum_{i=1}^n q_i^{[1]} = \sum_{i=1}^n q_i^{[2]} = 0, \nonumber \\
 & 0 \leq \beta_i^+, \beta_i^-, \beta_i^+ + \beta_i^- \leq D, \nonumber \\
 & 0 \leq \alpha _i^{{[1]}/{[2]}} \leq C_{1/2}, 
\end{align}
where $\alpha_i^{[1]}, \alpha_i^{[2]}, \beta_i^+, \beta_i^-$ are the vectors of Lagrangian multipliers, $k_1(x_i, x_j)=<\phi(x_i^{[1]}), \phi(x_j^{[1]})>$ and $k_2(x_i, x_j)=<\phi(x_i^{[2]}), \phi(x_j^{[2]})>$ are the kernel function corresponding to view 1 and view 2, respectively. The predictive function for each view is expressed as:
\begin{align}
    f_{{[1]}/{[2]}} = \sum_{i=1}^n q_i^{{[1]}/{[2]}}k_{1/2}(x_i, x) +b_{1/2}.
\end{align}
\section{Proposed Work}
\label{Proposed Work}
In this section, we integrate the wave loss (W-loss) function into the multiview support vector machine framework, proposing a novel approach called the multiview support vector machine with wave loss (Wave-MvSVM). We delineate the formulation and analysis of Wave-MvSVM to elucidate how it handles multi-view representation and noisy samples concurrently. By incorporating consensus and complementarity regularization terms, Wave-MvSVM effectively learns multi-view representations. To address noisy samples, Wave-MvSVM utilizes the asymmetry of the W-loss function to adaptively apply instance-level penalties to misclassified samples. Furthermore, it leverages the bounded nature of the wave loss function to prevent the model from excessively focusing on outliers.
\begin{figure*}[ht!]
\begin{minipage}{.5\linewidth}
\centering
\subfloat[$a=0.5$]{\label{fig:2a}\includegraphics[scale=0.4]{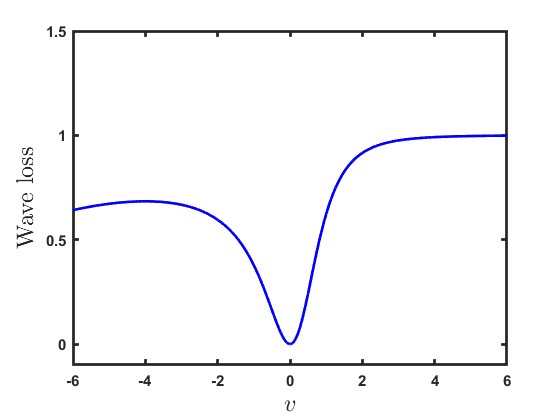}}
\end{minipage}
\begin{minipage}{.5\linewidth}
\centering
\subfloat[$a=1$]{\label{fig:2b}\includegraphics[scale=0.4]{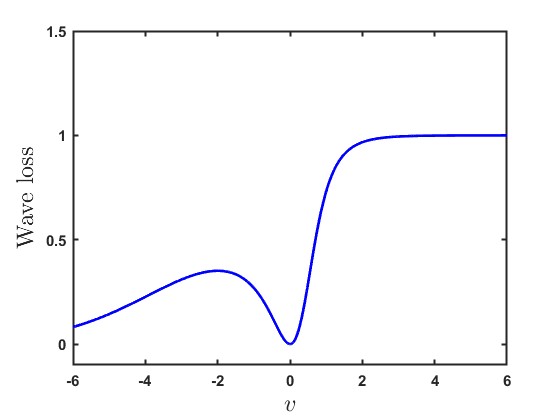}}
\end{minipage}
\par\medskip
\begin{minipage}{.5\linewidth}
\centering
\subfloat[$a=1.5$]{\label{fig:2c}\includegraphics[scale=0.4]{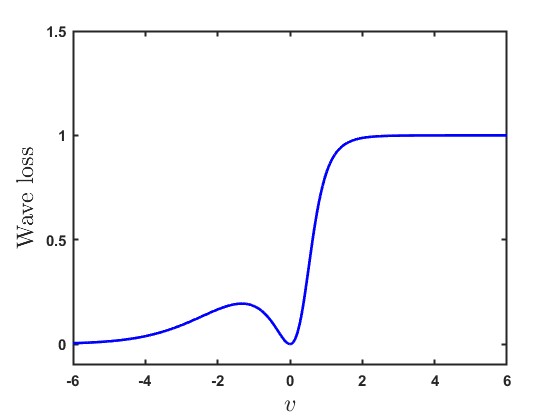}}
\end{minipage}
\begin{minipage}{.5\linewidth}
\centering
\subfloat[$a=2$]{\label{fig:2d}\includegraphics[scale=0.4]{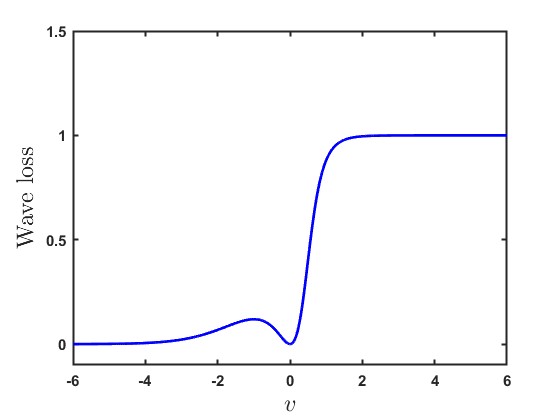}}
\end{minipage}
\caption{An illustration of the W-loss function, where \( \lambda \) is fixed at $1$, and different values of \( a \).}
\label{An illustration of the wave loss function}
\end{figure*}
\subsection{Wave loss function}
In this subsection, we present the formulation of the W-loss function \cite{akhtar2024advancing}, engineered to be resilient to outliers, robust against noise, and smooth in its properties, which can be formulated as:
\begin{align}
\label{Wave:1}
    \mathcal{L}_{wave}(h) = \frac{1}{\lambda}\left( 1 - \frac{1}{1+ \lambda h^2 \exp{(a h)}} \right ), \hspace{0.2cm} \forall \hspace{0.1cm} h \in \mathbb{R},
\end{align}
where \( \lambda \in \mathbb{R}^+ \) represents the bounding parameter and \( a \in \mathbb{R} \) denotes the shape parameter. Figure \ref{An illustration of the wave loss function} visually illustrates the W-loss function. The W-loss function exhibits the following properties:
\begin{enumerate}
    \item It is bounded, smooth, and asymmetric function.
    \item The W-loss function possesses two essential parameters: the shape parameter \( a \), dictating the shape of the loss function, and the bounding parameter \( \lambda \), determining the loss function threshold values. 
    \item The W-loss function is differentiable and hence continuous for all $h \in \mathbb{R}$.
    \item The W-loss function showcases resilience against outliers and noise insensitivity. With loss bounded to \( \frac{1}{\lambda} \), it handles outliers robustly, while assigning loss to samples with \( h \leq 0 \), displaying resilience to noise.
    \item As \( a \) tends to infinity, for a fixed \( \lambda \), the W-loss function converges point-wise to the \( 0 - \frac{1}{\lambda} \) loss, expressed as:
     \begin{align}
        \mathcal{L}_{0-1}(h) = \left\{
  \begin{array}{lr} 
      0 & h\leq 0, \\
      \frac{1}{\lambda} & h > 0. 
      \end{array}
    \right.
    \end{align}
    Furthermore, the W-loss converges to the ``$0 - 1$'' loss when \( \lambda = 1 \).
\end{enumerate}

\begin{figure*}
    \centering
    \includegraphics[width=0.8\textwidth,height=6.0cm]{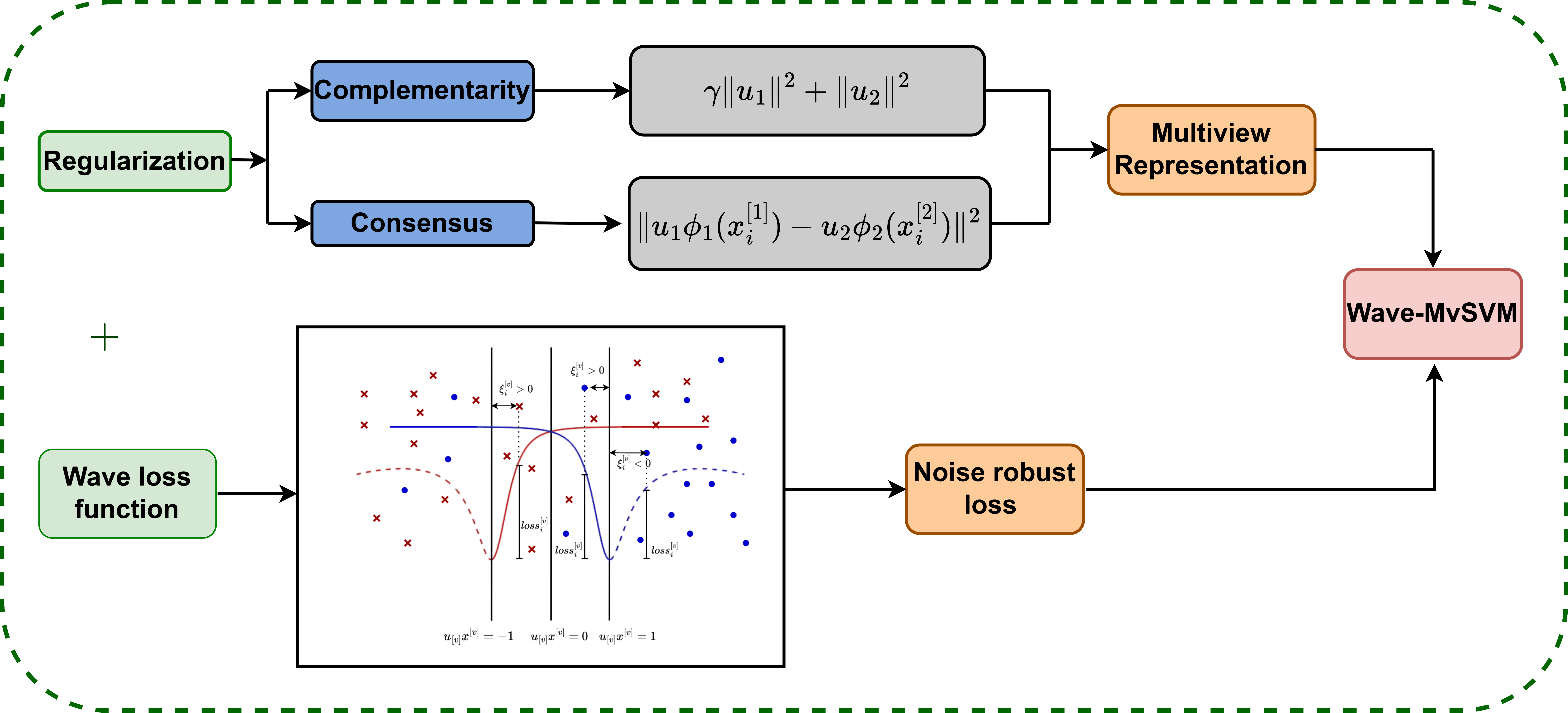}
    \caption{Flowchart of the model construction of Wave-MvSVM}
    \label{Flowchart of the model construction}
\end{figure*}
\subsection{Multiview support vector machine with W-loss (Wave-MvSVM)}
In this subsubsection, we present the multiview support vector machine with wave loss (Wave-MvSVM) that utilizes the consensus and complementarity principle. The flow diagram of the model construction is shown in Figure \ref{Flowchart of the model construction}. The optimization problem of Wave-MvSVM is given as follows:
\begin{align}
\label{PW1}
     \underset{u_1, u_2, \zeta^{[1]}, \zeta^{[2]}}{\min} & \frac{\gamma}{2} \|u_1\|^2 + \frac{1}{2} \|u_2\|^2    + \sum_{v=1}^{2}\sum_{i=1}^{n}\mathcal{C}_v \left( \frac{1}{\lambda_v}\left(1-\frac{1}{1+ \lambda_v \zeta^{[v]^2}_i\exp{(a_v\zeta^{[v]}_iy_i})}\right)\right) \nonumber \\
     & + \mathcal{D} \sum_{i=1}^n \left((u_1\phi_1(x^{[1]}_i)) - (u_2\phi_2(x^{[2]}_i)) \right)^2 \nonumber \\
    \text{s.t.} \hspace{0.5cm} & y_i (u_1\phi_1(x^{[1]}_i))  \geq 1 - \zeta^{[1]}_i,  \nonumber \\
    & y_i (u_2\phi_2(x^{[2]}_i))  \geq 1 - \zeta^{[2]}_i,  \nonumber \\
    & \zeta^{[1]}_i  \geq 0, \hspace{0.1cm} \zeta^{[2]}_i  \geq 0, \hspace{0.2cm} i=1, 2, \ldots, n,
\end{align}
where $\gamma, \mathcal{C}_1, \mathcal{C}_2, \mathcal{D}, a_1, a_2, \lambda_1, \lambda_2 > 0$ are tunable parameters, $u_1$ and $u_2$ represent the weight vectors for view $1$ and view $2$, $\phi_1$ and $\phi_2$ denote the feature mapping functions, and $\zeta^{[1]}_i$ and $\zeta^{[2]}_i$ are slack variables of view $1$ and view $2$ respectively.
\par
Each component of the optimization problem of Wave-MvSVM has the following significance:
\begin{enumerate}
    \item The terms $u_1$ and $u_2$ serve as regularization components for views 1 and 2, respectively. These terms are used to mitigate overfitting by limiting the capacities of the classifier sets for both views. The nonnegative regularization parameter \(\gamma\) determines the significance of each view to the final classifier, thus facilitating the exploration of complementary properties across different views.
    \item To encourage view agreement, we incorporate a regularization term between views in the objective function $\mathcal{D} \sum_{i=1}^n \left((u_1\phi_1(x^{[1]}_i)) - (u_2\phi_2(x^{[2]}_i)) \right)^2$. A smaller regularization term between views tends to produce more consistent learners from both perspectives.
    \item $\zeta_i^{[v]}$ represents a slack variable of the $v^{th}$ view, allowing Wave-MvSVM to accommodate misclassification. The W-loss function $\frac{e^T}{\lambda_v}\left(e-\frac{1}{e + \lambda_v \zeta^{[v]^2}_i\exp{(a_v\zeta^{[v]}_iy_i})}\right)$ is tailored to manage contaminated datasets with error-prone, noisy, and view-inconsistent samples. To elaborate, noisy samples are located in regions occupied by other classes, causing disturbances when determining the decision hyperplane.
\end{enumerate}
\begin{figure}[ht!]
    \centering
    \includegraphics[width=0.9\textwidth,height=6.5cm]{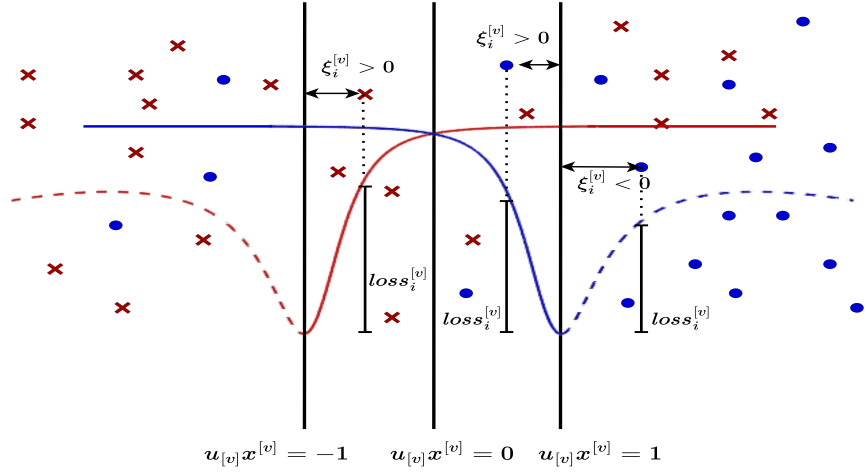}
    \caption{Geometrical depiction of Wave-MvSVM: This figure demonstrates how the model handles misclassified samples, the influence of slack variables, and the application of penalties for outliers and noisy data. It provides a visual understanding of the asymmetric and bounded nature of the W-loss function, showing its impact on the decision boundary and the overall robustness of the proposed model.}
    \label{Geometric representation Wave loss}
\end{figure}
Wave-MvSVM addresses noisy samples by leveraging the asymmetric and bounded nature of the W-loss function. The asymmetry allows the model to selectively impose penalties on misclassified samples at the instance level, while the boundedness ensures that noisy data does not overly influence the model's decisions. The slack variables $\zeta^{[1]}_i$ and $\zeta^{[2]}_i$ are integrated into the W-loss function, enabling it to assess and determine the misclassification costs for different samples dynamically. The cost of misclassification for the $i^{th}$ sample in view $1$ can be denoted as $loss^{[1]} = \frac{e^T}{\lambda_1}\left(e-\frac{1}{e + \lambda_1 \zeta^{[1]^2}_i\exp{(a_1\zeta^{[1]}_iy_i})}\right)$. To elucidate the mechanism of the W-loss function, we illustrate using the linear Wave-MvSVM model as depicted in Figure \ref{Geometric representation Wave loss}. It is worth mentioning that a comparable mechanism is also applicable to models with multiple views. In Figure \ref{Geometric representation Wave loss},  it is evident that for the $i^{th}$ sample, where $(x^{[1]}_i, +1)$ is misclassified to some extent, its corresponding slack variable $\zeta^{[1]}_i > 0$. In this scenario, we have $y_i \zeta^{[1]}_i = 1 \times \zeta^{[1]}_i > 0$, resulting in a  $loss^{[1]}_i = \frac{e^T}{\lambda_1}\left(e-\frac{1}{e + \lambda_1 \zeta^{[1]^2}_i\exp{(a_1\times 1 \times \zeta^{[1]}_i})}\right)$. Conversely, for the $j^{th}$ sample $(x^{[1]}_j, -1)$ exhibits a certain degree of misclassification, its corresponding slack variable $\zeta^{[1]}_j$ is greater than zero. Consequently, we have $y_j \zeta^{[1]}_j = -1 \times \zeta^{[1]}_j < 0$, resulting in $loss^{[1]}_j = \frac{e^T}{\lambda_1}\left(e-\frac{1}{e + \lambda_1 \zeta^{[1]^2}_i\exp{(a_1\times (-1) \times \zeta^{[1]}_i})}\right)$. When $a_1 = 0.5$, the shape of the W-loss function (as shown in Figure \ref{Geometric representation Wave loss}) indicates that $loss^{[1]}$ rises gradually with $\zeta^{[1]}$  and then decreases gradually after reaching a threshold. Moreover,  the W-loss function shows a gradual decline in misclassification as the degree of error exceeds a specific threshold. Beyond this threshold, when a positive class sample is severely misclassified, deviating significantly from the hyperplane, the model may perceive it as a noisy data point and assign a lighter penalty compared to a typical sample. This property enhances the model's resilience to outliers and noisy data. It is important to highlight that Wave-MvSVM focuses on minimizing loss for noisy samples within the positive class, which is more susceptible to containing noisy points due to its larger sample size. Furthermore, enforcing boundedness on the loss for the minority class poses a significant risk of losing valuable information, particularly considering its inherently limited representation. In summary, the W-loss function strikes a balance by appropriately penalizing misclassifications while avoiding excessive penalties for extreme misclassifications.
\section{Optimizaiton for Wave-MvSVM}
\label{Optimizaiton for Wave-MvSVM}
In this section, we use the ADMM and GD algorithms to optimize problem \eqref{PW1}. Initially, we rewrite problem \eqref{PW1} using the Representer Theorem \cite{dinuzzo2012representer}. Following this, we apply the ADMM and GD algorithms to solve the optimization problem (\ref{PW1}).
\par
According to the Representer Theorem, the solutions for \( u_1 \) and \( u_2 \) in equation \eqref{PW1} can be expressed as follows:
\begin{align}
\label{OP1}
    u_v = \sum_{j=1}^n \alpha^{[v]}_j \phi_v(x^{[v]}_j),
\end{align}
where $\alpha^{[v]} = (\alpha^{[v]}_1, \ldots, \alpha^{[v]}_n)^T$ , $v=\{1,2\}$, is the coefficient vector of view $1$ and view $2$. 
\par
Substituting \eqref{OP1} into \eqref{PW1}, we obtain:
\begin{align}
\label{OP2}
     \underset{u_1, u_2, \zeta^{[1]}, \zeta^{[2]}}{\min} & \frac{\gamma}{2} \alpha^{[1]^T}\mathcal{K}_1\alpha^{[1]} + \frac{1}{2} \alpha^{[2]^T}\mathcal{K}_2\alpha^{[2]}  + \mathcal{C}_1 \left( \frac{e^T}{\lambda_1}\left(e-\frac{1}{e+ \lambda_1 \zeta^{[1]^2}_i\exp{(a_1\zeta^{[1]}_iy_i})}\right)\right) \nonumber \\
     & + \mathcal{C}_2 \left( \frac{e^T}{\lambda_2}\left(e-\frac{1}{e+ \lambda_2 \zeta^{[2]^2}_i\exp{(a_2\zeta^{[2]}_iy_i})}\right)\right)  + \mathcal{D} e^T\left(\mathcal{K}_1\alpha^{[1]} - \mathcal{K}_2\alpha^{[2]} \right)^T \left(\mathcal{K}_1\alpha^{[1]} - \mathcal{K}_2\alpha^{[2]} \right) \nonumber \\ 
    \text{s.t.} \hspace{0.5cm} & Y\mathcal{K}_1\alpha^{[1]}  \geq 1 - \zeta^{[1]},  \nonumber \\
    & Y\mathcal{K}_2\alpha^{[2]}  \geq 1 - \zeta^{[2]},  \nonumber \\
    & \zeta^{[1]}  \geq 0, \hspace{0.1cm} \zeta^{[2]}  \geq 0,
\end{align}
where $e$ is a vector of ones, $\mathcal{K}_1=k_1(x^{[1]}_i, x^{[1]}_j) = (\phi_1(x^{[1]}_i), \phi_2(x^{[1]}_j))$ and $\mathcal{K}_2=k_1(x^{[2]}_i, x^{[2]}_j) = (\phi_1(x^{[2]}_i), \phi_2(x^{[2]}_j))$. $\zeta^{[1]}$ and $\zeta^{[2]}$ are the slack variables, respectively. \( Y = \text{diag}(y_1, y_2, \ldots, y_n) \) is the \( n \times n \) label matrix with \( y_1, y_2, \ldots, y_n \) as its diagonal entries. To solve the problem \eqref{OP2}, we employ the ADMM and GD algorithms to determine the optimal values of \( \alpha^{[1]} \), \( \alpha^{[2]} \), \( \zeta^{[1]} \), and \( \zeta^{[2]} \). The detailed derivation is given below.
\par
First, by introducing slack variables \(\eta_1, \eta_2, \eta_3, \eta_4 \in \mathbb{R}^n_+\), problem \eqref{OP2} is transformed into the following constrained problem:
\begin{align}
\label{OP3}
     \underset{u_1, u_2, \zeta^{[1]}, \zeta^{[2]}}{\min} & \frac{\gamma}{2} \alpha^{[1]^T}\mathcal{K}_1\alpha^{[1]} + \frac{1}{2} \alpha^{[2]^T}\mathcal{K}_2\alpha^{[2]} + \sum_{j=1}^4 \delta_{R_+^n} (\eta_j)  + \mathcal{C}_1 \left( \frac{e^T}{\lambda_1}\left(e-\frac{1}{e+ \lambda_1 \zeta^{[1]^2}_i\exp{(a_1\zeta^{[1]}_iy_i})}\right)\right) \nonumber \\
     & + \mathcal{C}_2 \left( \frac{e^T}{\lambda_2}\left(e-\frac{1}{e+ \lambda_2 \zeta^{[2]^2}_i\exp{(a_2\zeta^{[2]}_iy_i})}\right)\right) + \mathcal{D} e^T\left(\mathcal{K}_1\alpha^{[1]} - \mathcal{K}_2\alpha^{[2]} \right)^T \left(\mathcal{K}_1\alpha^{[1]} - \mathcal{K}_2\alpha^{[2]} \right) \nonumber \\ 
    \text{s.t.} \hspace{0.5cm} & Y\mathcal{K}_1\alpha^{[1]} - e + \zeta^{[1]} - \eta_1 =0,  \nonumber \\
    & Y\mathcal{K}_2\alpha^{[2]} - e + \zeta^{[2]} - \eta_2 = 0,  \nonumber \\
    & \zeta^{[1]} -\eta_3 = 0, \nonumber \\
    & \zeta^{[2]} - \eta_4 = 0,
\end{align}
where $\delta_{R_+^n} (\cdot)$ represents the indicator function on the positive real space $R_+^n$. The Lagrangian corresponding to the problem \eqref{OP3} is given by
\begin{align}
\label{OP4}
    \mathcal{L}(\alpha^{[1]}, \alpha^{[2]},& \zeta^{[1]}, \zeta^{[2]}, \eta_1, \eta_2, \eta_3, \eta_4, \vartheta_1, \vartheta_2, \vartheta_3, \vartheta_4) = \frac{1}{2} \left ( \gamma \alpha^{[1]^T}\mathcal{K}_1\alpha^{[1]} +  \alpha^{[2]^T}\mathcal{K}_2\alpha^{[2]} \right) + \sum_{j=1}^4 \delta_{R_+^n} (\eta_j) \nonumber \\
     & + \mathcal{C}_1 \left( \frac{e^T}{\lambda_1}\left(e-\frac{1}{e+ \lambda_1 \zeta^{[1]^2}_i\exp{(a_1\zeta^{[1]}_iy_i})}\right)\right)  + \mathcal{C}_2 \left( \frac{e^T}{\lambda_2}\left(e-\frac{1}{e+ \lambda_2 \zeta^{[2]^2}_i\exp{(a_2\zeta^{[2]}_iy_i})}\right)\right) \nonumber \\
      & + \mathcal{D} e^T\left(\mathcal{K}_1\alpha^{[1]} - \mathcal{K}_2\alpha^{[2]} \right)^T \left(\mathcal{K}_1\alpha^{[1]} - \mathcal{K}_2\alpha^{[2]} \right) + \vartheta_1^T \left( Y\mathcal{K}_1\alpha^{[1]} - e + \zeta^{[1]} - \eta_1 \right)  \nonumber \\
      & + \vartheta_2^T \left( Y\mathcal{K}_2\alpha^{[2]} - e + \zeta^{[2]} - \eta_2 \right)  + \frac{\kappa_1}{2} \|Y\mathcal{K}_1\alpha^{[1]} - e + \zeta^{[1]} - \eta_1\|_2^2  + \frac{\kappa_2}{2} \|Y\mathcal{K}_2\alpha^{[2]} - e + \zeta^{[2]} - \eta_2\|_2^2  \nonumber \\
      & + \vartheta_3^T(\zeta^{[1]} -\eta_3) +  \vartheta_4^T(\zeta^{[2]} -\eta_4) + \frac{\kappa_3}{2} \|\zeta^{[1]} -\eta_3\|_2^2  + \frac{\kappa_4}{2} \|\zeta^{[2]} -\eta_4\|_2^2,
\end{align}
where $\vartheta_1, \vartheta_2, \vartheta_3, \vartheta_4 >0$ are the vectors of Lagrangian multipliers and $\kappa_1, \kappa_2, \kappa_3, \kappa_4 >0$ are penalty parameters.
\par
For problem \eqref{OP2}, ADMM is used to derive the iterative steps for obtaining its optimal solution, as detailed below:
\begin{align}
    & \alpha^{[1]^{(t+1)}} = \underset{\alpha^{[1]}}{\arg\min} \mathcal{L}\left(\alpha^{[1]}, \alpha^{[2]^{(t)}}, \zeta^{[1]^{(t)}}, \zeta^{[2]^{(t)}}, \eta_1^{(t)}, \eta_2^{(t)}, \eta_3^{(t)}, \eta_4^{(t)}, \vartheta_1^{(t)}, \vartheta_2^{(t)}, \vartheta_3^{(t)}, \vartheta_4^{(t)} \right), \\
    & \alpha^{[2]^{(t+1)}} = \underset{\alpha^{[2]}}{\arg\min} \mathcal{L}\left(\alpha^{[1]^{(t+1)}}, \alpha^{[2]}, \zeta^{[1]^{(t)}}, \zeta^{[2]^{(t)}}, \eta_1^{(t)}, \eta_2^{(t)}, \eta_3^{(t)}, \eta_4^{(t)}, \vartheta_1^{(t)}, \vartheta_2^{(t)}, \vartheta_3^{(t)}, \vartheta_4^{(t)} \right), \\
     & \zeta^{[1]^{(t+1)}} = \underset{\zeta^{[1]}}{\arg\min} \mathcal{L}\left(\alpha^{[1]^{(t+1)}}, \alpha^{[2]^{(t+1)}}, \zeta^{[1]}, \zeta^{[2]^{(t)}}, \eta_1^{(t)}, \eta_2^{(t)},  \eta_3^{(t)}, \eta_4^{(t)}, \vartheta_1^{(t)}, \vartheta_2^{(t)}, \vartheta_3^{(t)}, \vartheta_4^{(t)} \right), \\
    & \zeta^{[2]^{(t+1)}} = \underset{\zeta^{[2]}}{\arg\min} \mathcal{L}\left(\alpha^{[1]^{(t+1)}}, \alpha^{[2]^{(t+1)}}, \zeta^{[1]^{(t+1)}}, \zeta^{[2]}, \eta_1^{(t)}, \eta_2^{(t)}, \eta_3^{(t)}, \eta_4^{(t)}, \vartheta_1^{(t)}, \vartheta_2^{(t)}, \vartheta_3^{(t)}, \vartheta_4^{(t)} \right), \\
    & \eta_1^{(t+1)} = \underset{\eta_1}{\arg\min} \mathcal{L}\left(\alpha^{[1]^{(t+1)}}, \alpha^{[2]^{(t+1)}}, \zeta^{[1]^{(t+1)}}, \zeta^{[2]^{(t+1)}}, \eta_1, \eta_2^{(t)}, \eta_3^{(t)}, \eta_4^{(t)}, \vartheta_1^{(t)}, \vartheta_2^{(t)}, \vartheta_3^{(t)}, \vartheta_4^{(t)} \right), \\
    & \eta_2^{(t+1)} = \underset{\eta_2}{\arg\min} \mathcal{L}\left(\alpha^{[1]^{(t+1)}}, \alpha^{[2]^{(t+1)}}, \zeta^{[1]^{(t+1)}}, \zeta^{[2]^{(t+1)}}, \eta_1^{(t+1)}, \eta_2, \eta_3^{(t)}, \eta_4^{(t)}, \vartheta_1^{(t)}, \vartheta_2^{(t)}, \vartheta_3^{(t)}, \vartheta_4^{(t)} \right), \\
    & \eta_3^{(t+1)} = \underset{\eta_3}{\arg\min} \mathcal{L}\left(\alpha^{[1]^{(t+1)}}, \alpha^{[2]^{(t+1)}}, \zeta^{[1]^{(t+1)}}, \zeta^{[2]^{(t+1)}}, \eta_1^{(t+1)}, \eta_2^{(t+1)}, \eta_3, \eta_4^{(t)}, \vartheta_1^{(t)}, \vartheta_2^{(t)}, \vartheta_3^{(t)}, \vartheta_4^{(t)} \right), \\
     & \eta_4^{(t+1)} = \underset{\eta_4}{\arg\min} \mathcal{L}\left(\alpha^{[1]^{(t+1)}}, \alpha^{[2]^{(t+1)}}, \zeta^{[1]^{(t+1)}}, \zeta^{[2]^{(t+1)}}, \eta_1^{(t+1)}, \eta_2^{(t+1)}, \eta_3^{(t+1)}, \eta_4, \vartheta_1^{(t)}, \vartheta_2^{(t)}, \vartheta_3^{(t)}, \vartheta_4^{(t)} \right), \\
     & \vartheta_1^{(t+1)} = \vartheta_1^{(t)} + \theta_1 \kappa_1 (Y\mathcal{K}_1\alpha^{[1]^{(t+1)}} - e + \zeta^{[1]^{(t+1)}} - \eta_1^{(t+1)}), \\
    & \vartheta_2^{(t+1)} = \vartheta_2^{(t)} + \theta_2 \kappa_2 (Y\mathcal{K}_2\alpha^{[2]^{(t+1)}} - e + \zeta^{[2]^{(t+1)}} - \eta_2^{(t+1)}), \\  
    & \vartheta_3^{(t+1)} = \vartheta_3^{(t)} + \theta_1 \kappa_3 (\zeta^{[1]^{(t+1)}} -\eta_3^{(t+1)}), \\
    & \vartheta_4^{(t+1)} = \vartheta_4^{(t)} + \theta_2 \kappa_4 (\zeta^{[2]^{(t+1)}} -\eta_4^{(t+1)}),
\end{align}
where $\theta_1, \theta_2 \in \left (0, \frac{\sqrt{5}+1}{2} \right)$ are the dual step lengths. Using the equation provided above, we can compute the closed-form solutions for the vectors $\alpha^{[1]}$, $\alpha^{[2]}$, Lagrange multiplier vectors $\vartheta_1, \vartheta_2, \vartheta_3, \vartheta_4$ and slack variables $\kappa_1, \kappa_2,\kappa_3, \kappa_4$. The optimal slack variables $\zeta^{[1]}$ and $\zeta^{[2]}$ can be determined by the GD algorithm keeping $\alpha^{[1]}, \alpha^{[2]}, \kappa_1, \kappa_2,\kappa_3, \kappa_4, \vartheta_1, \vartheta_2, \vartheta_3, \vartheta_4$ fixed. The objective function for the GD algorithm is given by:
\begin{align}
\label{OP18}
     \underset{\zeta^{[1]}, \zeta^{[2]}}{\min} \mathcal{L}_{\zeta}  & =\mathcal{C}_1 \left( \frac{e^T}{\lambda_1}\left(e-\frac{1}{e+ \lambda_1 \zeta^{[1]^2}_i\exp{(a_1\zeta^{[1]}_iy_i})}\right)\right) + \mathcal{C}_2 \left( \frac{e^T}{\lambda_2}\left(e-\frac{1}{e+ \lambda_2 \zeta^{[2]^2}_i\exp{(a_2\zeta^{[2]}_iy_i})}\right)\right) \nonumber \\ 
      & + \vartheta_1^T \left( Y\mathcal{K}_1\alpha^{[1]} - e + \zeta^{[1]} - \eta_1 \right) + \vartheta_2^T \left( Y\mathcal{K}_2\alpha^{[2]} - e + \zeta^{[2]} - \eta_2 \right)  + \frac{\kappa_1}{2} \|Y\mathcal{K}_1\alpha^{[1]} - e + \zeta^{[1]} - \eta_1\|_2^2  \nonumber \\
      & + \frac{\kappa_2}{2} \|Y\mathcal{K}_2\alpha^{[2]} - e + \zeta^{[2]} - \eta_2\|_2^2  + \vartheta_3^T(\zeta^{[1]} -\eta_3) +  \vartheta_4^T(\zeta^{[2]} -\eta_4) + \frac{\kappa_3}{2} \|\zeta^{[1]} -\eta_3\|_2^2  + \frac{\kappa_4}{2} \|\zeta^{[2]} -\eta_4\|_2^2.
\end{align}
Based on problem \eqref{OP18}, we can compute the gradients of $\mathcal{L}_{\zeta}$ w.r.t. $\zeta^{[1]}$ and $\zeta^{[2]}$, respectively.
\begin{align}
    & \nabla \mathcal{L}_{\zeta ^{[1]}} = \mathcal{C}_1\zeta^{[1]}\exp{(a_1Y\zeta^{[1]})} \left[ \frac{2e + \zeta^{[1]}a_1Y}{(e + \lambda_1\zeta^{[1]^2}\exp{(a_1Y\zeta^{[1]})} )^2}  \right]  + \vartheta_1  + \vartheta_3  \kappa_1 (Y\mathcal{K}_1\alpha^{[1]} - e + \zeta^{[1]} - \eta_1) + \kappa_3 + (\zeta^{[1]} -\eta_3), \\
   & \text{and} \nonumber \\
   & \nabla \mathcal{L}_{\zeta ^{[2]}} = \mathcal{C}_2\zeta^{[2]}\exp{(a_2Y\zeta^{[2]})} \left[ \frac{2e + \zeta^{[2]}a_2Y}{(e + \lambda_2\zeta^{[2]^2}\exp{(a_2Y\zeta^{[2]})} )^2}  \right]  + \vartheta_2  + \vartheta_4 + \kappa_2 (Y\mathcal{K}_2\alpha^{[2]} - e + \zeta^{[2]} - \eta_2) + \kappa_4 (\zeta^{[2]} -\eta_4).
\end{align}
The GD algorithm for solving \( \zeta^{[1]} \) and \( \zeta^{[2]} \) is outlined in the Algorithm \ref{GD algorithm for Wave-MvSVM}. Building on this, we outline the detailed iterative steps for the ADMM algorithm in the Algorithm \ref{ADMM algorithm for Wave-MvSVM}.
\begin{algorithm}[ht!]
\caption{GD algorithm for Wave-MvSVM}
\label{GD algorithm for Wave-MvSVM}
\textbf{Input:} Input dataset $\{(x_i^{[1]}, x_i^{[2]}, y_i)\}_{i=1}^n$, $y_i \in \{-1, 1\}$, W-loss parameters $a_1$, and $a_2$, penalty parameters $\mathcal{C}_1$, and $\mathcal{C}_2$, ADMM parameters $\kappa_1, \kappa_2, \kappa_3$, and $\kappa_4$ and maximum iteration number $Iter$. \\
\textbf{Output:} Optimal values $\zeta^{[1]^*}$ and $\zeta^{[2]^*}$ \\
\vspace{-0.4cm}
\begin{algorithmic}[1]
\STATE \textbf{Initialize:} $\zeta^{[1]^{(0)}}$, $\zeta^{[2]^{(0)}}$;
\STATE \textbf{while} not converge \textbf{do}
\STATE $\zeta^{[1]^{(t+1)}} = \zeta^{[1]^{(t)}} - \delta \nabla \mathcal{L}_{\zeta ^{[1]}}$;
\STATE $\zeta^{[2]^{(t+1)}} = \zeta^{[2]^{(t)}} - \delta \nabla \mathcal{L}_{\zeta ^{[2]}}$;
\STATE \textbf{end while}
\STATE \textbf{Return:} $\zeta^{[1]^*}$ and $\zeta^{[2]^*}$.
\end{algorithmic}
\end{algorithm}
\begin{algorithm}[ht!]
\caption{ADMM algorithm for Wave-MvSVM}
\label{ADMM algorithm for Wave-MvSVM}
\textbf{Input:} Input dataset $\{(x_i^{[1]}, x_i^{[2]}, y_i)\}_{i=1}^n$, $y_i \in \{-1, 1\}$, W-loss parameters $a_1, a_2$, penalty parameters $\mathcal{C}_1, \mathcal{C}_2$, ADMM parameters $\kappa_1, \kappa_2, \kappa_3, \kappa_4$ and maximum iteration number $Iter$. \\
\textbf{Output:} Model parameters
\begin{algorithmic}[1]
\STATE \textbf{Initialize:} $\alpha^{[1]^{(0)}}, \alpha^{[2]^{(0)}}, \zeta^{[1]^{(0)}}, \zeta^{[2]^{(0)}}, \eta_1^{(0)}, \eta_2^{(0)}, \eta_3^{(0)}, \eta_4^{(0)}$, $\vartheta_1^{(0)}, \vartheta_2^{(0)}, \vartheta_3^{(0)}, \vartheta_4^{(0)}$.\\
\STATE \textbf{while} not converge \textbf{do}\\
\STATE $\alpha^{[1]^{(t+1)}} = (\gamma_1\mathcal{K}_1 + 2 \mathcal{D}\mathcal{K}_1^T\mathcal{K}_1 + \kappa_1\mathcal{K}_1^T\mathcal{K}_1)^{-1} $ $(2\mathcal{D}\mathcal{K}_1^T\mathcal{K}_2\alpha^{[2]^{(t)}}- Y \mathcal{K}_1\vartheta_1^{(t)}+\kappa_1Y\mathcal{K}_1(e-\zeta^{[1]^{(t)}} + \eta_1^{[t]}))$; \\
\STATE $\alpha^{[2]^{(t+1)}} = (\gamma_2 \mathcal{K}_2 + 2 \mathcal{D}\mathcal{K}_2^T\mathcal{K}_2 + \kappa_2\mathcal{K}_2^T\mathcal{K}_2)^{-1} $ $(2\mathcal{D}\mathcal{K}_2^T\mathcal{K}_1\alpha^{[1]^{(t+1)}}- Y \mathcal{K}_2\vartheta_2^{(t)}+\kappa_2Y\mathcal{K}_2(e-\zeta^{[2]^{(t)}} + \eta_2^{[t]}))$; \\
\STATE $\zeta^{[1]^{(t+1)}} = \textbf{GD}(\mathcal{L} (\alpha^{[1]^{(t+1)}}, \alpha^{[2]^{(t+1)}}, \zeta^{[1]}, \zeta^{[2]^{(t)}}, \eta_1^{(t)}, \eta_2^{(t)}, \eta_3^{(t)}, \eta_4^{(t)}, \vartheta_1^{(t)}, \vartheta_2^{(t)}, \vartheta_3^{(t)}, \vartheta_4^{(t)})$; 
\STATE $\zeta^{[2]^{(t+1)}} = \textbf{GD}(\mathcal{L} (\alpha^{[1]^{(t+1)}}, \alpha^{[2]^{(t+1)}}, \zeta^{[1]^{(t+1)}}, \zeta^{[2]}, \eta_1^{(t)}, \eta_2^{(t)}, \eta_3^{(t)}, \eta_4^{(t)}, \vartheta_1^{(t)}, \vartheta_2^{(t)}, \vartheta_3^{(t)}, \vartheta_4^{(t)}))$;
\STATE $\eta_1^{(t+1)} = \prod_{R_+^n}(\vartheta_1^{(t)}/\kappa_1 + (Y\mathcal{K}_1\alpha^{[1]^{(t+1)}} - e + \zeta^{[1]^{(t+1)}}))$; 
\STATE $\eta_2^{(t+1)} = \prod_{R_+^n}(\vartheta_2^{(t)}/\kappa_2 + (Y\mathcal{K}_1\alpha^{[2]^{(t+1)}} - e + \zeta^{[2]^{(t+1)}}))$;
\STATE $\eta_3^{(t+1)} = \prod_{R_+^n}(\vartheta_3^{(t)}/\kappa_3 + \zeta^{[1]^{(t+1)}})$;
\STATE $\eta_4^{(t+1)} = \prod_{R_+^n}(\vartheta_4^{(t)}/\kappa_4 + \zeta^{[2]^{(t+1)}})$;
\STATE $\vartheta_1^{(t+1)} = \vartheta_1^{t} + \tau_1\kappa_1(Y\mathcal{K}_1\alpha^{[1]^{(t+1)}} - e + \zeta^{[1]^{(t+1)}} + \eta_1^{(t+1)})$;
\STATE $\vartheta_2^{(t+1)} = \vartheta_2^{t} + \tau_2\kappa_2(Y\mathcal{K}_2\alpha^{[2]^{(t+1)}} - e + \zeta^{[2]^{(t+1)}} + \eta_2^{(t+1)})$;
\STATE $\vartheta_3^{(t+1)} = \vartheta_3^{t} + \tau_1\kappa_3(\zeta^{[1]^{(t+1)}} - \eta_3^{(t+1)})$;
\STATE $\vartheta_4^{(t+1)} = \vartheta_4^{t} + \tau_2\kappa_4(\zeta^{[2]^{(t+1)}} - \eta_4^{(t+1)})$;
\STATE \textbf{end while}
\STATE \textbf{Return:} Optimal $\alpha^{[1]^*}$ and $\alpha^{[2]^*}$.
\end{algorithmic}
\end{algorithm}

After obtaining the optimal values \( \alpha^{[1]*} \) and \( \alpha^{[2]*} \) for the optimization problem \eqref{OP3}, the class label of a new sample \( (x^{[1]}, x^{[2]}) \) can be determined as follows:
\begin{align}
\label{eqq:27}
   g(x^{[1]},x^{[2]}) = sign(f(x^{[1]},x^{[2]})),
\end{align}
where
\begin{align}
    f(x^{[1]},x^{[2]}) = \gamma \sum_{i=1}^n \alpha^{[1]^*}_ik_1(x^{[1]}_i,x^{[1]}) + \sum_{i=1}^n \alpha^{[2]^*}_ik_2(x^{[2]}_i,x^{[2]}).
\end{align}
\section{Theoretical Analysis}
\label{Theoretical Analysis}
We discuss the complexity analysis of the proposed Wave-MvSVM model in subsection \ref{Computational complexity}. We delve into the theoretical analysis of the W-loss function in subsection \ref{Wave loss function theoretical analysis}. To thoroughly verify the generalization performance of Wave-MvSVM, in Section \ref{Generalization Capability Analysis}, we utilize Rademacher complexity theory to theoretically examine the generalization error bound and the consistency error bound of Wave-MvSVM.
\subsection{Computational complexity}
\label{Computational complexity}
In this subsection, we outline the computational complexity of Wave-MvSVM. The optimization problem for Wave-MvSVM is solved using ADMM and GD methods. Updating \( \alpha^{[1]} \) and \( \alpha^{[2]} \) with ADMM involves a computational complexity of \( \mathcal{O}(T_1 \times n^3) \), where \( n \) represents the number of samples and \( T_1 \) is the number of iterations. Updating \( \zeta^{[1]} \) and \( \zeta^{[2]} \) with GD has a computational complexity of \( \mathcal{O}(T_1 \times T_2 \times n^2) \), where \( T_2 \) denotes the number of iterations. The complexity for updating \( \eta_1 \), \( \eta_2 \), \( \vartheta_1 \), and \( \vartheta_2 \) using ADMM is \( \mathcal{O}(T_1 \times n^2) \), while updating \( \eta_3 \), \( \eta_4 \), \( \vartheta_3 \), and \( \vartheta_4 \) has a complexity of \( \mathcal{O}(T_1 \times n) \). Therefore, the overall complexities for the ADMM and GD algorithms are \( \mathcal{O}(T_1 \times n^3) + \mathcal{O}(T_1 \times n^2) + \mathcal{O}(T_1 \times n) \) and \( \mathcal{O}(T_1 \times T_2 \times n^2) \), respectively. Consequently, the total computational complexity of Wave-MvSVM is given by \( \mathcal{O}(T_1 \times n^3) + \mathcal{O}(T_1 \times n^2) + \mathcal{O}(T_1 \times n) + \mathcal{O}(T_1 \times T_2 \times n^2) \approx \mathcal{O}(T_1 \times n^3) + \mathcal{O}(T_1 \times T_2 \times n^2) \).
\subsection{W-loss function theoretical analysis}
\label{Wave loss function theoretical analysis}
In this subsection, we explore the theoretical foundations of the W-loss function, shedding light on its crucial classification-calibrated nature \cite{akhtar2024advancing}. Classification calibration, as introduced by \citet{bartlett2006convexity}, serves as a method for assessing the statistical effectiveness of loss functions. Classification calibration guarantees that the probabilities predicted by the model align closely with the actual likelihood of events, enhancing the reliability of the model's predictions. This aspect is crucial in deepening our understanding of how the W-loss function performs in classification tasks. 
\par
Suppose we have a training dataset \( \mathcal{X} = \{(x_i, y_i)\}_{i=1}^n \), where each sample \( x_i \) is paired with a label \( y_i \), and these samples are drawn independently from a probability distribution \( \mathcal{P} \). The probability distribution \(\mathcal{P}\) encompasses both the input space $X \subset \mathbb{R}^n$ and the corresponding label space $Y=\{-1, +1\}$. The main objective is to create a binary classifier \(\mathcal{E}\) that can do input processing in the \(X\) space and classify inputs into one of the labels in \(Y\). Generating a classifier that minimizes the related error is the main goal of this classification challenge. One can express the risk associated with a given classifier \(\mathcal{E}\) using the following formulation:
\begin{align}
    \mathcal{G}(\mathcal{E}) = \int_X \mathcal{P}(\mathcal{E}(x) \neq y|x)d\mathcal{P}_X.
\end{align}
Here, the probability distribution of the label \(y\) given an input \(x\) is represented by \(\mathcal{P}(y|x)\), and \(d\mathcal{P}_X\) denotes the marginal distribution of the input \(x\). Furthermore, the conditional distribution \(\mathcal{P}(y|x)\) is binary, implying it is governed by \(\mathcal{P}(x) = \text{Prob}(y = 1|x)\) and \(1 - \mathcal{P}(x) = \text{Prob}(y = -1|x)\). In essence, the goal is to create a classifier that precisely categorizes inputs, reducing the overall classification error.
\par
When $\mathcal{P}(x)$ is not equal to $\frac{1}{2}$, the Bayes classifier can be described as follows:
\begin{align}
        \mathcal{F}_{\mathcal{E}}(x) = \left\{
  \begin{array}{lr} 
      -1 & \mathcal{P}(x) < \frac{1}{2}, \\
      1 & \mathcal{P}(x) > \frac{1}{2}. 
      \end{array}
    \right.
\end{align}
It has been shown that the Bayes classifier minimizes classification error optimally. This can be expressed mathematically as: 
\begin{align}
     \mathcal{F}_{\mathcal{E}} = \underset{\mathcal{E}:X \rightarrow Y}{\arg\min} \mathcal{G}(\mathcal{E}).
\end{align}
For a given loss function $\mathcal{L}$, the expected error of a classifier $\mathcal{F}:X \rightarrow \mathbb{R}$ can be formulated as follows:
\begin{align}
    \mathcal{G}_{\mathcal{L}, \mathcal{P}}(\mathcal{F}) = \int_{X \times Y} \mathcal{L}(1-y\mathcal{F}(x))d\mathcal{P}. 
\end{align}
$\mathcal{F}_{\mathcal{L}, \mathcal{P}}$ aimed at minimizing the expected error across all feasible functions, is expressed as follows:
\begin{align}
    \mathcal{F}_{\mathcal{L}, \mathcal{P}}(x) = \underset{\mathcal{F}(x) \in \mathbb{R}}{\arg\min} \int_{Y} \mathcal{L}(1-y\mathcal{F}(x))d\mathcal{P}(y|x), ~~ \forall ~~ x \in X.
\end{align}
The W-loss function $\mathcal{L}_{\text{wave}}(\cdot)$ satisfies Theorem \ref{TH12}, demonstrating its classification-calibrated nature \cite{akhtar2024advancing}. This property guarantees that minimizing the expected error corresponds to the sign of the Bayes classifier, underscoring the importance of the W-loss function.
\begin{theorem} 
\label{TH12}
\cite{akhtar2024advancing}: The W-loss function $\mathcal{L}_{\text{wave}}(v)$ is characterized by classification calibration, ensuring that $\mathcal{F}_{\mathcal{L}_{\text{wave}},\mathcal{P}}$ shares the same sign as the Bayes classifier.
\end{theorem}
\begin{proof}
    We get the following outcome from a simple calculation:
    \begin{align}
         \int_{Y} \mathcal{L}_{wave}&(1-y\mathcal{F}(x))d\mathcal{P}(y|x)  = \mathcal{L}_{wave}(1+\mathcal{F}(x))(1- \mathcal{P}(x)) + \mathcal{L}_{wave}(1-\mathcal{F}(x)) \mathcal{P}(x) \nonumber \\
       & = \frac{1}{\eta}\left( 1 - \frac{1}{1+\eta(1+\mathcal{F}(x))^2e^{\gamma(1+\mathcal{F}(x))}} \right) (1- \mathcal{P}(x)) +  \frac{1}{\eta}\left( 1 - \frac{1}{1+\eta(1-\mathcal{F}(x))^2e^{\gamma(1-\mathcal{F}(x))}} \right) \mathcal{P}(x).
    \end{align}
The graphical depiction of $\int_{Y} \mathcal{L}_{\text{wave}}(1-y\mathcal{F}(x))d\mathcal{P}(y|x)$ as functions of $\mathcal{F}(x)$ are shown in Figures \ref{figw1} and \ref{figw2}, respectively, for the situations where $\mathcal{P}(x) > 1/2$ and $\mathcal{P}(x) < 1/2$. Figure \ref{wave Thm} demonstrates that the minimum value of $\int_{Y} \mathcal{L}_{\text{wave}}(y-\mathcal{F}(x))d\mathcal{P}(y|x)$ corresponds to a positive $\mathcal{F}(x)$ when $\mathcal{P}(x)>1/2$. In contrast, when $\mathcal{P}(x) < 1/2$, the minimum value corresponds to a negative value of $\mathcal{F}(x)$.
\par
Therefore, the trends depicted in Figures \ref{figw1} and \ref{figw2} confirm that the W-loss function $\mathcal{L}_{\text{wave}}(v)$ is classification-calibrated.
\end{proof}
\begin{figure*}[ht!]
\begin{minipage}{.50\linewidth}
\centering
\subfloat[]{\label{figw1} \includegraphics[scale=0.4]{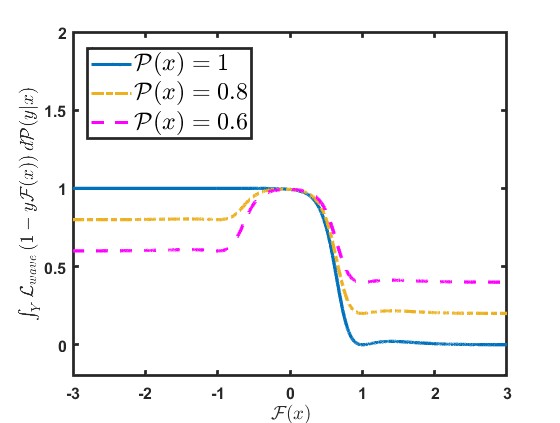}}
\end{minipage}
\begin{minipage}{.50\linewidth}
\centering
\subfloat[]{\label{figw2} \includegraphics[scale=0.4]{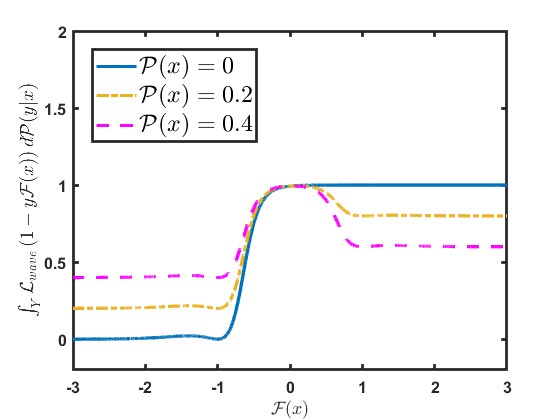}}
\end{minipage}
\caption{Graphically represent $\int_{Y} \mathcal{L}_{\text{wave}}(1-y\mathcal{F}(x))d\mathcal{P}(y|x)$ as a function of $\mathcal{F}(x)$, considering various values of $\mathcal{P}(x)$.
(a) Illustrate the case where $\mathcal{P}(x) > 1/2$.
(b) Depict the case when $\mathcal{P}(x)<1/2$.}
\label{wave Thm}
\end{figure*}
\subsection{Generalization capability analysis}
\label{Generalization Capability Analysis}
In this subsection, we introduce the generalization error bound for Wave-MvSVM. To start, we define the Rademacher complexity \cite{bartlett2002rademacher}:
\begin{definition}
\label{def1}
Let $\mathcal{Q}$ represent the probability distribution over the set $T$, where $T = \{x_1, \ldots, x_n\}$ consists of $n$ independent samples drawn from $\mathcal{Q}$. For a function set $\mathscr{G}$ defined on $T$, the empirical Rademacher complexity on $\mathscr{G}$ is defined as:
\begin{align}
    \hat{R}_n(\hat{\mathscr{G}}) & = \mathbb{E}_{\sigma}\left[ \underset{g \in \mathscr{G}}{\sup} | \frac{2}{n} \sum_{i=1}^n \sigma_i g(x_i) |: x_1, x_2,\ldots, x_n \right],
\end{align}
where \( \sigma = (\sigma_1, \ldots, \sigma_n) \) are independent Rademacher random variables taking values \(\{+1,-1\}\). The Rademacher complexity of $\mathscr{G}$ is given by:
\begin{align}
    R_n(\mathscr{G}) & = \hat{\mathbb{E}}_T[\hat{R_n}(\mathscr{G})] =\hat{\mathbb{E}}_{T\sigma}\left[ \underset{g \in \mathscr{G}}{\sup} | \frac{2}{n} \sum_{i=1}^n \sigma_i g(x_i) |\right].
\end{align}
\end{definition}
\begin{lemma}
\label{lemma1}
    Choose \( \theta \) from the interval $(0, 1)$ and consider $\mathscr{G}$ as a class of functions mapping from an input space $T$ to $[0, 1]$. Suppose \( \{x_i\}_{i=1}^n \) be independently drawn from a probability distribution $\mathcal{Q}$. Then, with a probability of at least \( 1 - \theta \) over random samples of size $n$, every $g \in \mathscr{G}$ satisfies:
    \begin{align}
        \mathbb{E}_{\mathcal{Q}}[g(x)] \leq \hat{\mathbb{E}}_{\mathcal{Q}}[g(x)] + \hat{R}_n(\mathscr{G}) + 3\sqrt{\frac{\ln(2/\theta)}{2n}}.
    \end{align}
\end{lemma}
\begin{lemma}
\label{lemma2}
    let $T = \{(x_i, y_i)\}_{i=1}^n$ be a sample set from \( X \) and assume $\mathcal{K}: X \times X \rightarrow \mathbb{R}$ is a kernel. For the function class $\mathscr{G}_B = \{g|g: x \rightarrow (u\cdot \phi(x)): \|u\| \leq B\}$ in the kernel feature space, the empirical Rademacher complexity of $\mathscr{G}_B$ satisfies:
   \begin{align}
       \hat{R}_n (\mathscr{G}_B) = \frac{2B}{n} \sqrt{\sum_{i=1}^n \mathcal{K}(x_i, x_i)}.
   \end{align}
\end{lemma}
\begin{lemma}
\label{lemma3}
Consider $\mathscr{A}$  as a Lipschitz function with a Lipschitz constant $\mathcal{L}$, mapping the real numbers to the real numbers, satisfying $\mathscr{A}(0) = 0$. The Rademacher complexity of the class $\mathscr{A} \circ \mathscr{G}$ is:
\begin{align}
    \hat{R}_n (\mathscr{A} \circ \mathscr{G}) \leq 2\mathcal{L}\hat{R}_n (\mathscr{G}).
\end{align}
\end{lemma}
As shown in equation \eqref{eqq:27}, we utilize the weighted predictions from the two views as the prediction function in Wave-MvSVM. Consequently, we can derive the generalization error bound of Wave-MvSVM using the following theorem.
\begin{theorem}
Given $N \in \mathbb{R}^+$, $\theta \in (0, 1)$, and a training set $T = \{(x_i, y_i)\}_{i=1}^n$ drawn  independently and identically from probability distribution $\mathcal{Q}$, where $y_i \in \{-1, +1\}$ and $x_i = (x_i^A; \delta x_i^B)$. Define the function class classes $\mathscr{G} = \{g|g: x \rightarrow u^T\hat{\phi}(x), \|u\| \le N\}$ and $\hat{\mathscr{G}} = \{\hat{g}|\hat{g}: (x, y) \rightarrow yg(x), g(x) \in \mathscr{G}\}$, where $u = (u_1; u_2)$, $\hat{\phi}(x)=(\phi_1(x_i^{[1]}); \phi_2(x_i^{[2]}))$ and $g(x) = \left(u_1^T\phi_1(x^{[1]}) +  \delta u_2^T\phi_2(x^{[2]}) \right) = u^T\hat{\phi}(x) \in \mathscr{G}$. Then, with a probability of at least $1 - \theta$ over $T$, every $g(x) \in \mathscr{G}$ satisfies
\begin{align}
\label{Th1}
    \mathbb{P}_{\mathcal{Q}}[yg(x)\leq0] \le \frac{1}{n(1+\delta)} \sum_{i=1}^n(\zeta_i^{[1]} + \delta\zeta_i^{[2]}) + 3\sqrt{\frac{\ln(2/\theta)}{2n}}  + \frac{4N}{n(1+\delta)}\sqrt{\sum_{i=1}^n(\mathcal{K}_1(x_i^{[1]}, x_i^{[1]})+ \delta^2 \mathcal{K}_2(x_i^{[2]}, x_i^{[2]}))}.
\end{align}
\end{theorem}
\begin{proof}
Let's consider a loss function $\Omega: \mathbb{R} \rightarrow [0, 1]$ defined as:
\begin{align}
    \Omega(x)=\left\{\begin{array}{lll}{1}, & \text{if}~~ x < 0,  \vspace{3mm} \\ {1-\frac{x}{1+\delta}} , & \text{if}~~ 0 \leq x \leq 1+\delta, \\
    {0}, & otherwise.
 \end{array}\right. 
\end{align}
Then, we have
\begin{align}
\label{TH2:3}
    \mathbb{P}_{\mathcal{Q}}[yg(x)\leq0] \leq \mathbb{E}_{\mathcal{Q}}[\Omega(\hat{g}(x, y))].
\end{align}
Using Lemma \ref{lemma1}, we have
\begin{align}
    \mathbb{E}_{\mathcal{Q}}[\Omega(\hat{g}(x, y))-1] &\leq \hat{\mathbb{E}}_{T}[\Omega(\hat{g}(x, y))-1] + 3\sqrt{\frac{\ln(2/\theta)}{2n}} + \hat{R}_n((\Omega-1)\circ \mathscr{G}).
\end{align}
Therefore,
\begin{align}
\label{TH2:4}
    \mathbb{E}_{\mathcal{Q}}[\Omega(\hat{g}(x, y))] \leq & \hat{\mathbb{E}}_{T}[\Omega(\hat{g}(x, y))] + 3\sqrt{\frac{\ln(2/\theta)}{2n}} + \hat{R}_n((\Omega-1)\circ \mathscr{G}).  
\end{align}
From the first and second constraints in Wave-MvSVM, we deduce:
\begin{align}
\label{TH2:5}
    \hat{\mathbb{E}}_{T}[\Omega(\hat{g}(x, y))] & \leq \frac{1}{n(1+\delta)} \sum_{i=1}^n [1 + \delta - y_ig(x_i)]_+ \nonumber \\
    &  = \frac{1}{n(1+\delta)} \sum_{i=1}^n [1 - y_ig_1(x_i^{[1]}) + \delta(1 - y_ig_2(x_i^{[2]})) ]_+  \nonumber \\
    &  \leq \frac{1}{n(1+\delta)} \sum_{i=1}^n \{[1 - y_ig_1(x_i^{[1]})]_+  + \delta[1 - y_ig_2(x_i^{[2]})]_+ \} \nonumber \\
    & \leq \frac{1}{n(1+\delta)} \sum_{i=1}^n (\zeta_i^{[1]} + \delta\zeta_i^{[2]}).
\end{align}
Considering that $(\Omega - 1)(x)$ is a Lipschitz function with a constant of $\frac{1}{1 + \delta}$, which intersects the origin and is uniformly bounded, we can establish the following inequality based on Lemma \ref{lemma3}: 
\begin{align}
    \hat{R}_n((\Omega-1) \circ \hat{\mathscr{G}}) \leq \frac{2}{1+\delta} \hat{R}_n(\hat{\mathscr{G}}).
\end{align}
Using the Definition \ref{def1}, we obtain:
\begin{align}
    \hat{R}_n(\hat{\mathscr{G}}) & = \mathbb{E}_{\sigma}\left[ \underset{\hat{g} \in \hat{\mathscr{G}}}{\sup} | \frac{2}{n} \sum_{i=1}^n \sigma_i \hat{g}(x_i, y_i)  \right] \nonumber \\
    & = \mathbb{E}_{\sigma}\left[ \underset{g \in \mathscr{G}}{\sup} | \frac{2}{n} \sum_{i=1}^n \sigma_i y_ig(x_i)  \right] \nonumber \\
    & = \mathbb{E}_{\sigma}\left[ \underset{g \in \mathscr{G}}{\sup} | \frac{2}{n} \sum_{i=1}^n \sigma_i g(x_i)  \right] \nonumber \\
    & = \hat{R}_n(\mathscr{G}).
\end{align}
Combining this with Lemma \ref{lemma2}, we have:
\begin{align}
\label{TH2:6}
     \hat{R}_n((\Omega-1) \circ \hat{\mathscr{G}})&\leq \frac{2}{1+\delta} \hat{R}_n(\mathscr{G})  \nonumber \\
    & = \frac{4N}{n(1+\delta)}\sqrt{\sum_{i=1}^n \hat{\mathcal{K}}(x_i, x_i)} \nonumber \\
    &  = \frac{4N}{n(1+\delta)}\sqrt{\sum_{i=1}^n (\mathcal{K}_1(x_i^{[1]}, x_i^{[1]}) + \delta^2 \mathcal{K}_2(x_i^{[2]}, x_i^{[2]}))}.
\end{align}
Moreover, by combining equations \eqref{TH2:3}, \eqref{TH2:4}, \eqref{TH2:5}, and \eqref{TH2:6}, we can obtain inequality, which demonstrates the generalization error bound of Wave-MvSVM. 
\end{proof}
We define the classification error function \( g(x) \) using the integrated decision function defined in Wave-MvSVM (\ref{PW1}). By integrating the empirical Rademacher complexity of \( \mathcal{G} \) with the empirical expectation of \( \hat{\mathcal{G}} \), we establish a margin-based estimate of the misclassification probability. As \( n \) increases significantly, Wave-MvSVM demonstrates a robust generalization error bound for classification tasks. As the training error decreases, the generalization error also decreases accordingly. This theoretically ensures that Wave-MvSVM achieves superior generalization performance.
\section{Numerical Experiments}
\label{Experiments and Results}
To assess the effectiveness of the proposed Wave-MvSVM model, we evaluate their performance against baseline models including SVM-2K \cite{farquhar2005two}, MvTwSVM \cite{xie2015multi}, MVNPSVM \cite{tang2018multi}, PSVM-2V \cite{tang2017multiview}, MVLDM \cite{hu2024multiview} and MVCSKL \cite{tang2023multi}.
We conduct experiments on publicly available benchmark datasets, which include $30$ real-world UCI \cite{dua2017uci} and KEEL \cite{derrac2015keel} datasets, as well as $45$ binary classification datasets obtained from the Animal with Attributes (AwA)\footnote{\url{http://attributes.kyb.tuebingen.mpg.de}} dataset.
\subsection{Experimental setup}
The experimental hardware setup consists of a PC featuring an Intel(R) Xeon(R) Gold $6226$R CPU operating at $2.90$GHz, with $128$ GB of RAM, and running the Windows $11$ operating system. The experiments are conducted using Matlab R$2023$a. The dataset is randomly partitioned into a $70:30$ ratio, allocating $70\%$ of the data for training and $30\%$ for testing. We utilize a five-fold cross-validation technique along with a grid search approach to optimize the hyperparameters of the models. For all experiments, we opt for the Gaussian kernel function represented by $\mathcal{K}(x_i, x_j) = e^{-\frac{\|x_i - x_j\|^2}{2\sigma^2}}$ for each model. The kernel parameter $\sigma$ is selected from the following range: $\{ 2^{-5}, 2^{-4}, \ldots, 2^4, 2^5\}$. The parameters $\mathcal{C}_1$, $\mathcal{C}_2$, and $\mathcal{D}$ are selected from $\{ 2^{-2}, 2^{-1}, 1, 2^1, 2^2\}$. W-loss function parameters $\lambda_1$ and $\lambda_2$ are selected from the range $\{0.2,0.4,0.6,0.8,1\}$, and the other parameters $a_1$ and $a_2$ are selected from the range $\{0,0.5,1,1.5,2\}$. For the baseline MvTSVM and MVNPSVM model, we set \( C_1 = C_2 = C_3 = C_4 \) selected from the range $\{ 2^{-5}, 2^{-4}, \ldots, 2^4, 2^5\}$. In SVM-2K and PSVM-2V, we set \( C_1 = C_1 = D \) and selected from the range $\{ 2^{-5}, 2^{-4}, \ldots, 2^4, 2^5\}$. The parameter \( \epsilon \) in the proposed Wave-MvSVM model along with the baseline models is set to $0.1$. The parameters \( a_1 = a_2 \), \( d \), \( \gamma \), \( c_1 \), \( c_2 \) in MVCSKL are tuned in $\{ 2^{-2}, 2^{-1}, 1, 2^1, 2^2\}$. For MVLDM, the parameter $C_1, v_1, v_2, \theta, \sigma$ are chosen from $\{ 2^{-5}, 2^{-4}, \ldots, 2^4, 2^5\}$.
\subsection{Experiments on real-world UCI and KEEL datasets}
\begin{table*}[ht!]
\centering
    \caption{Classification performance of proposed Wave-MvSVM model along with the baseline models on UCI and KEEL datasets with the Gaussian kernel.}
    \label{Classification performance of UCI in nonLinear Case.}
    \resizebox{0.9\textwidth}{!}{
\begin{tabular}{lccccccc}
\hline
Dataset & SVM2K \cite{farquhar2005two} & MvTSVM \cite{xie2015multi} & PSVM-2V \cite{tang2017multiview} & MVNPSVM \cite{tang2018multi} & MVLDM \cite{hu2024multiview} & MVCSKL \cite{tang2023multi} & Wave-MvSVM$^{\dagger}$ \\
 & $(C_1, \sigma)$ & $(C_1, \sigma)$ & $(C_1, \gamma, \sigma)$ & $(C_1, D, \sigma)$ & $(C_1, v_1, v_2, \theta, \sigma)$ & $(a_1, d, \gamma, c_1, \sigma)$ & $(a_1, \lambda, \mathcal{C}_1, _2, D, \sigma)$ \\ \hline
abalone9-18 & $93.52$ & $92.31$ & $93.61$ & $92.89$ & $93.15$ & $\underline{94.06}$ & $\textbf{94.15}$ \\
 & $(16, 0.125)$ & $(0.5, 32)$ & $(0.03125, 0.03125, 0.03125)$ & $(1, 0.4, 8)$ & $(2, 0.25, 0.25, 4, 0.25)$ & $(0.1, 2, 0.25, 0.25, 4)$ & $(0.5, 0.2, 4, 0.25, 0.25, 4)$ \\
aus & $81.16$ & $82.61$ & $83.06$ & $80.34$ & $\underline{87.92}$ & $74.88$ & $\textbf{88.68}$ \\
 & $(0.125, 32)$ & $(0.03125, 32)$ & $(0.5, 0.5, 32)$ & $(0.0625, 0.3, 4)$ & $(0.25, 0.25, 4, 2, 2)$ & $(0.1, 0.25, 0.25, 0.25, 0.25)$ & $(0, 0.2, 1, 0.25, 0.25, 1)$ \\
blood & $78.57$ & $75.95$ & $78.13$ & $79.56$ & $\textbf{80.80}$ & $77.68$ & $\underline{80.68}$ \\
 & $(1, 0.5)$ & $(2, 16)$ & $(0.125, 0.5, 1)$ & $(0.25, 0.1, 32)$ & $(1, 1, 0.25, 0.25, 1)$ & $(0.1, 0.25, 0.25, 0.25, 0.25)$ & $(0, 0.2, 0.25, 0.25, 0.25, 0.25)$ \\
breast\_cancer\_wisc\_diag & $96.24$ & $92.73$ & $92.06$ & $92.35$ & $\underline{97.06}$ & $63.53$ & $\textbf{97.88}$ \\
 & $(8, 32)$ & $(0.03125, 32)$ & $(0.25, 0.03125, 2)$ & $(0.03125, 0.1, 2)$ & $(0.5, 0.5, 1, 4, 0.5)$ & $(0.1, 0.5, 0.25, 0.25, 0.25)$ & $(0.5, 0.4, 2, 2, 4, 0.25)$ \\
breast\_cancer\_wisc\_prog & $71.36$ & $71.38$ & $67.97$ & $71.02$ & $\underline{81.36}$ & $77.97$ & $\textbf{82.88}$ \\
 & $(2, 8)$ & $(0.03125, 32)$ & $(32, 1, 16)$ & $(0.25, 0.1, 32)$ & $(0.25, 0.25, 0.25, 0.25, 0.25)$ & $(1, 0.5, 0.5, 0.25, 0.25)$ & $(2, 1, 0.25, 0.25, 0.25, 4)$ \\
brwisconsin & $\underline{97.04}$ & $95.82$ & $96.08$ & $83.82$ & $95.59$ & $92.25$ & $\textbf{97.75}$ \\
 & $(0.5, 32)$ & $(0.03125, 1)$ & $(0.125, 0.03125, 4)$ & $(0.03125, 0.1, 0.125)$ & $(0.25, 0.25, 0.25, 0.25, 0.25)$ & $(10, 0.25, 0.25, 0.25, 2)$ & $(0, 0.2, 0.25, 4, 4, 0.5)$ \\
bupa or liver-disorders & $\underline{71.67}$ & $69.42$ & $\textbf{71.84}$ & $66.99$ & $67.96$ & $61.17$ & $71.05$ \\
 & $(0.5, 32)$ & $(16, 32)$ & $(0.5, 0.03125, 32)$ & $(0.03125, 0.1, 0.03125)$ & $(0.25, 0.25, 0.25, 0.25, 0.25)$ & $(0.1, 2, 2, 0.25, 0.25)$ & $(0, 0.2, 0.5, 4, 0.25, 0.25)$ \\
checkerboard\_Data & $81.16$ & $79.3$ & $85.44$ & $85.44$ & $\underline{85.51}$ & $78.74$ & $\textbf{87.01}$ \\
 & $(8, 8)$ & $(0.0625, 16)$ & $(0.0625, 1, 8)$ & $(0.03125, 0.1, 4)$ & $(0.25, 0.25, 0.25, 0.25, 0.25)$ & $(0.1, 0.25, 0.25, 0.25, 0.25)$ & $(0, 0.6, 0.25, 4, 0.25, 0.25)$ \\
cleve & $87.64$ & $\underline{81.73}$ & $80.9$ & $74.16$ & $\textbf{82.02}$ & $79.78$ & $81.65$ \\
 & $(0.25, 32)$ & $(0.03125, 4)$ & $(4, 0.5, 32)$ & $(0.03125, 0.6, 32)$ & $(0.5, 0.5, 0.25, 0.25, 0.5)$ & $(1, 0.5, 0.5, 0.25, 0.5)$ & $(0, 0.2, 0.25, 0.25, 0.25, 2)$ \\
congressional\_voting & $61.54$ & $55.41$ & $59.23$ & $56.15$ & $62.31$ & $\underline{66.15}$ & $\textbf{66.23}$ \\
 & $(0.5, 2)$ & $(0.03125, 4)$ & $(0.03125, 0.5, 4)$ & $(0.03125, 0.6, 0.5)$ & $(0.25, 0.25, 0.25, 0.25, 0.25)$ & $(0.1, 0.25, 0.25, 0.5, 0.25)$ & $(1, 0.6, 0.5, 0.25, 0.25, 0.25)$ \\
conn\_bench\_sonar\_mines\_rocks & $\underline{88.71}$ & $71.23$ & $90.32$ & $64.52$ & $82.26$ & $79.03$ & $\textbf{90.32}$ \\
 & $(2, 16)$ & $(0.03125, 8)$ & $(16, 0.03125, 16)$ & $(0.25, 0.7, 8)$ & $(0.25, 0.25, 0.25, 0.25, 0.25)$ & $(10, 1, 1, 0.25, 0.25)$ & $(2, 0.4, 2, 0.25, 4, 2)$ \\
credit\_approval & $85.99$ & $76.4$ & $85.51$ & $83.09$ & $\underline{86.47}$ & $69.57$ & $\textbf{86.53}$ \\
 & $(0.125, 16)$ & $(0.03125, 8)$ & $(0.03125, 0.0625, 4)$ & $(0.0625, 0.5, 4)$ & $(0.25, 0.25, 2, 4, 0.25)$ & $(0.1, 4, 4, 0.25, 0.25)$ & $(0, 0.2, 4, 0.25, 0.25, 4)$ \\
cylinder\_bands & $73.2$ & $72.7$ & $73.2$ & $62.75$ & $\underline{77.78}$ & $58.17$ & $\textbf{78.55}$ \\
 & $(16, 8)$ & $(0.03125, 8)$ & $(2, 0.03125, 8)$ & $(0.03125, 0.4, 8)$ & $(1, 1, 0.25, 0.25, 1)$ & $(0.1, 2, 2, 0.25, 0.25)$ & $(1.5, 0.6, 0.25, 0.25, 0.25, 1)$ \\
fertility & $86.67$ & $85.57$ & $73.33$ & $86.67$ & $\underline{90}$ & $83.33$ & $\textbf{90.67}$ \\
 & $(0.03125, 0.03125)$ & $(0.03125, 0.03125)$ & $(0.03125, 0.03125, 1)$ & $(2, 0.1, 32)$ & $(0.25, 0.25, 0.25, 0.25, 0.25)$ & $(1, 0.25, 0.25, 0.25, 0.25)$ & $(0.5, 1, 1, 0.25, 0.25, 1)$ \\
hepatitis & $\underline{84.78}$ & $80.73$ & $\underline{84.78}$ & $78.26$ & $\underline{82.61}$ & $76.09$ & $\textbf{86.09}$ \\
 & $(1, 8)$ & $(0.03125, 0.03125)$ & $(1, 1, 8)$ & $(0.03125, 0.8, 16)$ & $(0.25, 0.25, 0.25, 4, 0.25)$ & $(1, 0.25, 0.25, 0.25, 0.25)$ & $(0.5, 0.2, 0.5, 0.25, 0.25, 4)$ \\
molec\_biol\_promoter & $\underline{83.87}$ & $72$ & $67.74$ & $74.52$ & $\underline{83.87}$ & $74.19$ & $\textbf{84.19}$ \\
 & $(1, 32)$ & $(0.03125, 32)$ & $(32, 0.5, 32)$ & $(1, 0.6, 16)$ & $(0.25, 0.25, 0.25, 0.25, 0.25)$ & $(1, 0.25, 0.25, 0.25, 0.25)$ & $(1, 0.4, 0.5, 0.25, 0.25, 0.5)$ \\
monks\_1 & $80.18$ & $80.38$ & $80.77$ & $75.66$ & $84.94$ & $\underline{89.76}$ & $\textbf{91.54}$ \\
 & $(32, 4)$ & $(0.03125, 8)$ & $(8, 0.03125, 2)$ & $(0.0625, 0.2, 8)$ & $(4, 4, 0.25, 0.25, 4)$ & $(10, 0.25, 0.25, 0.25, 0.25)$ & $(0, 0.2, 0.5, 0.25, 2, 0.25)$ \\
monks\_2 & $\textbf{77.33}$ & $73.16$ & $75$ & $71.11$ & $75.56$ & $66.11$ & $\underline{75.78}$ \\
 & $(16, 1)$ & $(0.03125, 0.25)$ & $(1, 0.03125, 1)$ & $(16, 0.9, 16)$ & $(2, 2, 0.25, 0.25, 2)$ & $(1, 0.25, 0.25, 0.25, 0.25)$ & $(0, 0.2, 1, 0.25, 0.25, 1)$ \\
monks\_3 & $90.96$ & $88.66$ & $93.37$ & $71.69$ & $93.37$ & $\underline{93.98}$ & $\textbf{95.57}$ \\
 & $(16, 8)$ & $(4, 16)$ & $(8, 0.5, 4)$ & $(0.03125, 0.1, 8)$ & $(1, 1, 0.25, 0.25, 1)$ & $(1, 4, 4, 0.25, 0.25)$ & $(1, 0.4, 0.5, 0.25, 0.25, 4)$ \\
musk\_1 & $85.85$ & $83.53$ & $81.55$ & $86.34$ & $\underline{88.73}$ & $86.62$ & $\textbf{91.55}$ \\
 & $(4, 32)$ & $(0.03125, 16)$ & $(4, 0.03125, 16)$ & $(0.25, 0.9, 32)$ & $(0.25, 0.25, 0.25, 0.25, 0.25)$ & $(0.1, 0.25, 0.25, 0.25, 4)$ & $(0, 0.8, 0.5, 0.25, 0.25, 4)$ \\
planning & $70.37$ & $\underline{73.44}$ & $70.37$ & $69.26$ & $72.22$ & $66.67$ & $\textbf{74.07}$ \\
 & $(1, 1)$ & $(0.03125, 0.03125)$ & $(0.125, 0.03125, 2)$ & $(0.125, 0.1, 0.125)$ & $(0.25, 0.25, 0.25, 0.25, 0.25)$ & $(0.1, 2, 2, 0.25, 0.25)$ & $(1, 0.8, 0.25, 0.25, 0.25, 1)$ \\
shuttle-6\_vs\_2-3 & $97.1$ & $96$ & $96.67$ & $96.89$ & $93.98$ & $\textbf{98.55}$ & $\underline{97.20}$ \\
 & $(0.03125, 0.03125)$ & $(0.03125, 0.03125)$ & $(0.03125, 0.03125, 16)$ & $(0.125, 0.1, 16)$ & $(0.25, 0.25, 0.25, 0.25, 0.25)$ & $(0.1, 0.25, 0.25, 0.25, 0.25)$ & $(2, 0.8, 1, 0.25, 0.25, 0.25)$ \\
sonar & $80.48$ & $79.45$ & $79.03$ & $80.65$ & $\underline{81.90}$ & $80.16$ & $\textbf{82.26}$ \\
 & $(2, 0.25)$ & $(0.03125, 32)$ & $(8, 0.25, 4)$ & $(1, 0.9, 8)$ & $(4, 4, 0.25, 4, 4)$ & $(0.1, 0.5, 0.5, 0.25, 0.25)$ & $(2, 1, 0.25, 0.25, 0.25, 0.5)$ \\
statlog\_heart & $75.72$ & $73.54$ & $77.01$ & $77.65$ & $\underline{84.56}$ & $\textbf{85.31}$ & $\textbf{85.31}$ \\
 & $(0.5, 16)$ & $(0.03125, 4)$ & $(4, 0.03125, 32)$ & $(0.03125, 0.7, 8)$ & $(0.25, 0.25, 0.25, 0.25, 0.25)$ & $(10, 1, 1, 0.25, 0.25)$ & $(0, 0.2, 1, 0.25, 0.25, 1)$ \\
tic\_tac\_toe & $\textbf{100}$ & $98.51$ & $97.21$ & $98.26$ & $\underline{98.94}$ & $\textbf{100}$ & $\textbf{100}$ \\
 & $(4, 4)$ & $(0.03125, 32)$ & $(0.5, 0.03125, 4)$ & $(0.5, 0.9, 32)$ & $(0.25, 0.25, 0.25, 0.25, 0.25)$ & $(0.1, 2, 2, 0.25, 0.25)$ & $(2, 0.8, 0.5, 0.25, 0.25, 4)$ \\
vehicle & $\underline{77.05}$ & $73.93$ & $72.49$ & $76.68$ & $76.46$ & $74.31$ & $\textbf{78.28}$ \\
 & $(16, 32)$ & $(0.03125, 2)$ & $(0.03125, 2, 2)$ & $(0.03125, 0.1, 0.03125)$ & $(0.25, 0.25, 0.25, 0.25, 0.25)$ & $(10, 0.5, 0.5, 0.25, 2)$ & $(1.5, 1, 2, 4, 4, 0.25)$ \\
vertebral\_column\_2clases & $\underline{84.95}$ & $\textbf{86.18}$ & $76.89$ & $80.65$ & $75.4$ & $77.42$ & $82.37$ \\
 & $(2, 1)$ & $(0.125, 16)$ & $(16, 16, 2)$ & $(0.5, 0.3, 2)$ & $(2, 2, 0.25, 0.25, 2)$ & $(10, 0.5, 0.5, 0.5, 4)$ & $(0, 0.2, 0.25, 0.25, 0.25, 1)$ \\
votes & $93.08$ & $\underline{93.77}$ & $92.15$ & $\textbf{94.62}$ & $90.58$ & $92.31$ & $92.31$ \\
 & $(4, 16)$ & $(0.03125, 8)$ & $(16, 0.03125, 4)$ & $(2, 0.3, 8)$ & $(2, 2, 0.25, 0.25, 2)$ & $(0.1, 1, 1, 0.25, 0.25)$ & $(1, 0.2, 1, 0.25, 0.25, 0.25)$ \\
vowel & $\textbf{100}$ & $\underline{98.99}$ & $98.89$ & $95.27$ & $97.82$ & $\textbf{100}$ & $\textbf{100}$ \\
 & $(2, 1)$ & $(0.03125, 2)$ & $(4, 0.03125, 4)$ & $(0.03125, 0.1, 0.03125)$ & $(0.25, 0.25, 0.25, 0.25, 0.25)$ & $(0.1, 1, 1, 0.25, 0.25)$ & $(1, 0.2, 0.5, 0.25, 0.25, 4)$ \\
wpbc & $72.41$ & $\underline{75.74}$ & $72.69$ & $70.69$ & $61.01$ & $72.41$ & $\textbf{75.86}$ \\
 & $(1, 16)$ & $(0.03125, 16)$ & $(1, 0.125, 32)$ & $(0.03125, 0.3, 32)$ & $(0.25, 0.25, 0.25, 0.25, 0.25)$ & $(0.1, 2, 2, 0.25, 2)$ & $(1.5, 0.8, 0.25, 0.25, 0.25, 0.25)$ \\ \hline
Average ACC & $83.62$ & $81.02$ & $81.58$ & $79.27$ & $\underline{83.74}$ & $79.67$ & $\textbf{86.21}$ \\ \hline
Average Rank & $\underline{3.43}$ & $4.87$ & $4.7$ & $5.17$ & $3.57$ & $4.72$ & $\textbf{1.55}$ \\ \hline
 \multicolumn{8}{l}{$^{\dagger}$ represents the proposed models.}\\
 \multicolumn{8}{l}{The boldface and underline indicate the best and second-best models, respectively, in terms of ACC.}
\end{tabular}}
\end{table*}
In this section, we conduct a thorough analysis that includes comparing the proposed Wave-MvSVM model with baseline models across $30$ UCI \cite{dua2017uci} and KEEL \cite{derrac2015keel} benchmark datasets. Since the UCI and KEEL datasets do not inherently possess multiview characteristics, we designate the $95\%$ principal component extracted from the original data as view $2$, while referring to the original data as view $1$ \cite{wang2023safe}. The performance of the proposed Wave-MvSVM model, along with the baseline models, is evaluated using accuracy (ACC) metrics as shown in Table \ref{Classification performance of UCI in nonLinear Case.}. Optimal hyperparameters of the proposed Wave-MvSVM model along with the baseline models are shown in Table \ref{Classification performance of UCI in nonLinear Case.}. The average ACC of the proposed Wave-MvSVM model along with the baseline SVM2K, MvTSVM, PSVM-2V, MVNPSVM, MVLDM, and MVCSKL models are $86.21\%$ $83.62\%$, $81.02\%$, $81.58\%$, $79.27\%$, $83.74\%$, and $79.67$, respectively. In terms of average ACC, the proposed Wave-MvSVM achieved the top position. This indicates that the proposed Wave-MvSVM models display a high level of confidence in their predictive capabilities. Average ACC can be misleading because it may mask a model's superior performance on one dataset by compensating for its inferior performance on another. To mitigate the limitations of average ACC and determine the significance of the results, we utilized a suite of statistical tests recommended by \citet{demvsar2006statistical}. These tests are specifically designed for comparing classifiers across multiple datasets, particularly when the conditions necessary for parametric tests are unmet. We employed the following tests: ranking test, Friedman test, and Nemenyi post hoc test. By incorporating statistical tests, our aim is to comprehensively evaluate the performance of the models, enabling us to draw broad and unbiased conclusions regarding their effectiveness. In the ranking scheme, each model is assigned a rank based on its performance on individual datasets, enabling an assessment of its overall performance. Higher ranks are assigned to the worst-performing models, while lower ranks are attributed to the best-performing models. By employing this methodology, we account for the potential compensatory effect, where superior performance on one dataset offsets inferior performance on others. For the evaluation of $p$ models across $N$ datasets, the rank of the $j^{th}$ model on the $i^{th}$ dataset can be denoted as $\mathcal{R}_{j}^i$. Then the average rank of the $j^{th}$ model is given by $\mathcal{R}_j = \frac{1}{N} \sum_{i=1}^{N} \mathcal{R}_{j}^i$. The rank of the proposed Wave-MvSVM model along with the baseline SVM2K, MvTSVM, PSVM-2V, MVNPSVM, MVLDM, and MVCSKL models are $1.55$, $3.43$, $4.87$, $4.70$, $5.17$, $3.57$, and $4.72$, respectively. The Wave-MvSVM model achieved an average rank of $1.55$, which is the lowest among all the models. Given that a lower rank signifies a better-performing model, the proposed Wave-MvSVM model emerged as the top-performing model. The Friedman test \cite{friedman1937use}, compares whether significant differences exist among the models by comparing their average ranks. The Friedman test, a nonparametric statistical analysis, is employed to compare the effectiveness of multiple models across diverse datasets. Under the null hypothesis, the models' average rank is equal, implying that they perform equally well. The Friedman test follows the chi-squared distribution $(\chi^2_F)$ with $(p - 1)$ degrees of freedom (d.o.f), and its calculation involves: $\chi_F^2 = \frac{12N}{p(p+1)} \left[ \sum_j\mathcal{R}_j^2 - \frac{p(p+1)^2}{4} \right]$. The $F_F$ statistic is calculated as: $F_F = \frac{(N-1)\chi_F^2}{N(p-1)-\chi_F^2}$, where $F$-distribution has $(p-1)$ and $(N-1)\times(p-1)$. For $N=30$ and $p=7$, we obtained $\chi_F^2 = 62.5275$ and $F_F = 15.4359$. Referring to the \( F \)-distribution table with a significance level of \( 5\% \), we find \( F(6, 174) = 2.1510 \). As $F_F > 2.1510$, the null hypothesis is rejected. Hence, notable discrepancies are evident among the models. Consequently, we proceed to utilize the Nemenyi post hoc test \cite{demvsar2006statistical} to evaluate the pairwise differences among the models. The critical difference ($C.D.$) is calculated as $C.D. = q_{\alpha} \times \sqrt{\frac{p(p+1)}{6N}}$. Here, $q_\alpha$ denotes the critical value obtained from the distribution table for the two-tailed Nemenyi test. Referring to the statistical F-distribution table, where $q_\alpha = 2.949$ at a $5\%$ significance level, the $C.D.$ is computed as $1.645$. The average rank differences between the proposed Wave-MvSVM model with the baseline SVM2K, MvTSVM, PSVM-2V, MVNPSVM, MVLDM, and MVCSKL models are $1.88$, $3.32$, $3.15$, $3.62$, $2.02$, and $3.17$. The Nemenyi post hoc test validates that the proposed Wave-MvSVM model exhibits statistically significant superiority compared to the baseline models. We conclude that the proposed Wave-MvSVM model outperforms the existing models in terms of overall performance and ranking.
\begin{figure*}[ht!]
\begin{minipage}{.5\linewidth}
\centering
\subfloat[hepatitis]{\includegraphics[scale=0.4]{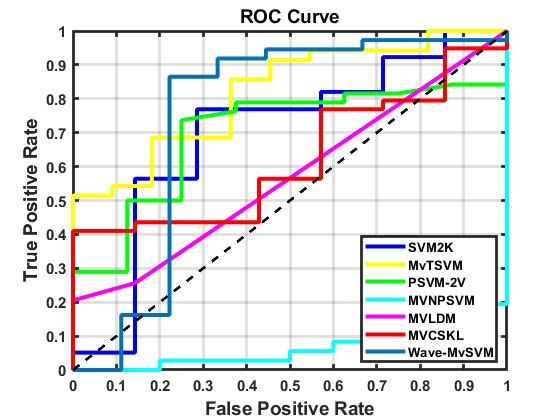}}
\end{minipage}
\begin{minipage}{.5\linewidth}
\centering
\subfloat[monks\_3]{\includegraphics[scale=0.4]{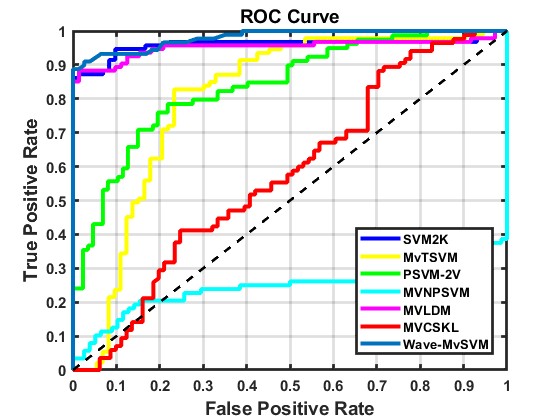}}
\end{minipage}
\par\medskip
\begin{minipage}{.5\linewidth}
\centering
\subfloat[planning]{\includegraphics[scale=0.4]{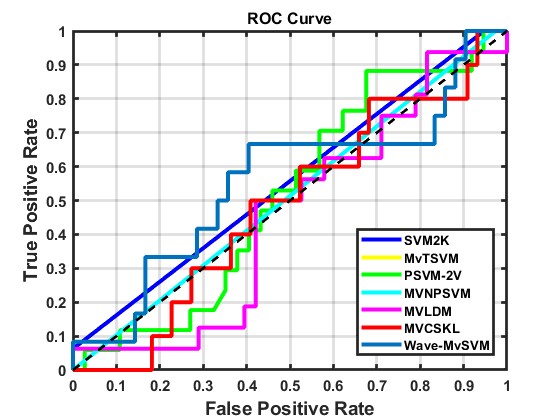}}
\end{minipage}
\begin{minipage}{.5\linewidth}
\centering
\subfloat[monks\_1]{\includegraphics[scale=0.4]{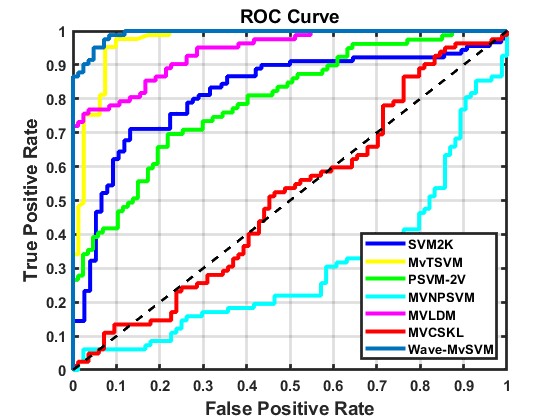}}
\end{minipage}
\caption{ROC curves of the proposed Wave-MvSVM model along with the baseline models on the UCI and KEEL datasets.}
\label{ROC Curve}
\end{figure*}
Figure \ref{ROC Curve} presents the ROC curve, demonstrating the superior performance of the proposed Wave-MvSVM model compared to the baseline models on the UCI and KEEL datasets. The ROC curve provides a detailed assessment of the model's diagnostic capability by plotting the true positive rate against the false positive rate across different threshold settings. The area under the ROC curve (AUC) for the proposed Wave-MvSVM model is notably higher, reflecting a better trade-off between sensitivity and specificity. This higher AUC indicates that the wave-MvSVM model excels at distinguishing between positive and negative instances, leading to more precise predictions. The proposed Wave-MvSVM model enhances its ability to correctly identify true positives, thus minimizing false negatives and ensuring more dependable detection. These results highlight the robustness and effectiveness of the proposed Wave-MvSVM model in classification tasks, surpassing the performance of the baseline models.
\begin{figure}[ht!]
\centering
{\label{cong2}\includegraphics[scale=0.45]{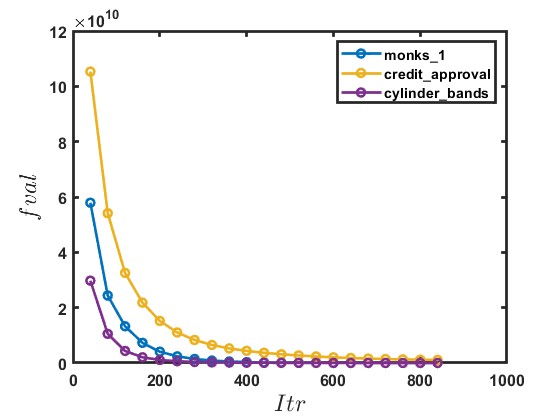}}
\caption{Convergence analysis of the proposed Wave-MvSVM model on monk\_1, credit\_approval, and cylinder\_bands datasets. The horizontal axis represents the number of iterations, while the vertical axis denotes the value of the objective function}
\label{Convergence analysis of AwA}
\end{figure}
\subsection{Convergence analysis}
\label{Convergence analysis}
We examined the convergence of Wave-MvSVM on three benchmark datasets, namely monk\_1, credit\_approval, and cylinder\_bands. The convergence analysis, depicted in Figure \ref{Convergence analysis of AwA}, illustrates that the objective function decreases monotonically with increasing iterations, stabilizing after approximately $200$ iterations. The examination revealed that the objective function consistently decreased with each iteration, indicating a steady and reliable path toward minimization. This continuous decline highlights the efficacy of the ADMM algorithm in optimizing the Wave-MvSVM model.
From Figure \ref{Convergence analysis of AwA}, the rapid convergence observed, with the objective function stabilizing within approximately $200$ iterations, highlights the robustness and efficiency of the ADMM algorithm. Based on these findings, we conclude that it is both reasonable and appropriate to utilize the ADMM algorithm for solving Wave-MvSVM, ensuring efficient and reliable convergence of the objective function.
\begin{table*}[ht!]
\centering
    \caption{Classification performance of proposed Wave-MvSVM model along with the baseline models on AwA datasets with the Gaussian kernel.}
    \label{Classification performance of AwA in nonLinear Case.}
    \resizebox{0.9\textwidth}{!}{
\begin{tabular}{lccccccc}
\hline
Dataset & SVM2K \cite{farquhar2005two} & MvTSVM \cite{xie2015multi} & PSVM-2V \cite{tang2017multiview} & MVNPSVM \cite{tang2018multi} & MVLDM \cite{hu2024multiview} & MVCSKL \cite{tang2023multi} & Wave-MvSVM$^{\dagger}$ \\
 & $(C_1, \sigma)$ & $(C_1, \sigma)$ & $(C_1, \gamma, \sigma)$ & $(C_1, D, \sigma)$ & $(C_1, v_1, v_2, \theta, \sigma)$ & $(a_1, d, \gamma, c_1, \sigma)$ & $(a_1, \lambda, \mathcal{C}_1, _2, D, \sigma)$ \\ \hline
 Chimpanzee vs Persian cat & $77.5$ & $68.33$ & $70.83$ & $70$ & $\underline{86.11}$ & $73.61$ & $\textbf{88.53}$ \\
 & $(4, 16)$ & $(0.03125, 32)$ & $(0.03125, 4, 8)$ & $(0.03125, 0.1, 1)$ & $(0.5, 0.25, 4, 0.25, 0.25, 4)$ & $(0.1, 0.25, 0.25, 0.25, 0.25)$ & $(2, 0.6, 0.25, 0.25, 2, 0.5)$ \\
Chimpanzee vs Pig & $\underline{72.50}$ & $66.67$ & $71.67$ & $68.33$ & $66.67$ & $68.06$ & $\textbf{72.88}$ \\
 & $(4, 16)$ & $(0.03125, 4)$ & $(1, 0.25, 1)$ & $(0.125, 0.2, 1)$ & $(0.25, 0.25, 0.25, 0.25, 0.25, 2)$ & $(1, 2, 2, 2, 4)$ & $(2, 0.8, 0.25, 0.25, 2, 0.5)$ \\
Chimpanzee vs Giant panda & $63.33$ & $62.5$ & $68.33$ & $66.67$ & $\underline{72.22}$ & $70.04$ & $\textbf{74.14}$ \\
 & $(2, 1)$ & $(0.03125, 1)$ & $(0.0625, 16, 4)$ & $(4, 0.8, 8)$ & $(2, 0.25, 0.25, 0.25, 4, 2)$ & $(10, 0.25, 0.25, 0.25, 4)$ & $(1, 0.6, 0.25, 0.25, 0.25, 4)$ \\
Chimpanzee vs Leopard & $70$ & $69.17$ & $\textbf{72.50}$ & $69.83$ & $68.75$ & $70$ & $\underline{72.22}$ \\
 & $(2, 0.5)$ & $(0.03125, 32)$ & $(0.03125, 32, 16)$ & $(4, 0.9, 8)$ & $(0.25, 0.25, 0.25, 0.25, 0.5, 1)$ & $(0.1, 1, 1, 0.25, 0.5)$ & $(0, 0.6, 0.25, 0.25, 1, 2)$ \\
Chimpanzee vs Hippopotamus & $74.83$ & $64.17$ & $70.83$ & $74.83$ & $\underline{78.47}$ & $76.39$ & $\textbf{78.83}$ \\
 & $(2, 32)$ & $(0.03125, 32)$ & $(0.03125, 2, 1)$ & $(32, 0.8, 1)$ & $(0.5, 0.5, 0.25, 0.25, 0.5, 2)$ & $(0.1, 0.25, 0.25, 0.25, 0.25)$ & $(0.5, 0.8, 0.25, 0.25, 0.25, 4)$ \\
 Chimpanzee vs Rat & $70.17$ & $69.67$ & $70$ & $70$ & $68.06$ & $\underline{70.83}$ & $\textbf{70.89}$ \\
 & $(1, 4)$ & $(0.03125, 32)$ & $(0.03125, 4, 1)$ & $(0.0625, 0.1, 1)$ & $(1, 1, 0.25, 0.25, 1, 1)$ & $(1, 0.5, 0.5, 4, 4)$ & $(0, 0.6, 0.25, 0.25, 2, 2)$ \\
Chimpanzee vs Seal & $71.33$ & $70$ & $71.17$ & $70.17$ & $75.69$ & $\underline{81.94}$ & $\textbf{85.42}$ \\
 & $(2, 16)$ & $(0.03125, 32)$ & $(1, 0.25, 1)$ & $(0.125, 0.2, 2)$ & $(0.25, 0.25, 0.25, 0.25, 0.25, 2)$ & $(1, 2, 2, 2, 4)$ & $(0, 0.8, 0.25, 0.25, 2, 0.5)$ \\
Chimpanzee vs Humpback whale & $82.17$ & $72.5$ & $82.5$ & $70.33$ & $81.25$ & $\underline{92.36}$ & $\textbf{96.53}$ \\
 & $(4, 2)$ & $(0.03125, 32)$ & $(0.25, 0.0625, 2)$ & $(0.03125, 0.1, 1)$ & $(2, 2, 0.25, 4, 2, 2)$ & $(0.1, 0.25, 0.25, 0.25, 0.25)$ & $(1, 0.6, 0.25, 4, 0.25, 4)$ \\
Chimpanzee vs Raccoon & $70.83$ & $64.17$ & $72.5$ & $68.33$ & $72.22$ & $\underline{72.92}$ & $\textbf{74.83}$ \\
 & $(4, 32)$ & $(0.03125, 32)$ & $(0.03125, 1, 2)$ & $(8, 0.9, 4)$ & $(4, 4, 0.25, 0.25, 4, 2)$ & $(1, 4, 4, 0.25, 4)$ & $(0, 0.8, 4, 0.25, 1, 4)$ \\
Giant panda vs Leopard & $60.33$ & $60.5$ & $60.83$ & $60.67$ & $61.81$ & $\textbf{79.86}$ & $\underline{77.08}$ \\
 & $(4, 1)$ & $(0.03125, 32)$ & $(0.0625, 16, 2)$ & $(8, 0.9, 4)$ & $(1, 2, 2, 4, 0.25, 1)$ & $(1, 0.25, 0.25, 4, 0.25)$ & $(0, 0.8, 0.25, 0.25, 0.5, 4)$ \\
Giant panda vs Persian cat & $\textbf{79.50}$ & $76.67$ & $76.67$ & $\underline{79.17}$ & $66.67$ & $77.78$ & $78.83$ \\
 & $(2, 0.5)$ & $(0.03125, 32)$ & $(2, 8, 2)$ & $(0.0625, 0.1, 1)$ & $(4, 0.25, 0.25, 0.25, 4, 2)$ & $(0.1, 1, 1, 0.25, 0.25)$ & $(0, 0.8, 0.25, 4, 0.5, 0.25)$ \\
Giant panda vs Pig & $60.83$ & $65$ & $61.67$ & $63.33$ & $\underline{65.97}$ & $57.64$ & $\textbf{68.75}$ \\
 & $(2, 0.5)$ & $(8, 32)$ & $(0.5, 0.25, 1)$ & $(0.0625, 0.1, 0.5)$ & $(0.25, 0.25, 4, 4, 0.5, 2)$ & $(0.1, 0.25, 0.25, 0.25, 0.25)$ & $(2, 0.4, 0.25, 0.25, 0.25, 4)$ \\
Giant panda vs Hippopotamus & $67.5$ & $70$ & $\textbf{78.33}$ & $70.67$ & $74.31$ & $73.61$ & $\underline{75.69}$ \\
 & $(1, 1)$ & $(0.03125, 32)$ & $(0.125, 32, 4)$ & $(0.03125, 0.1, 0.5)$ & $(1, 0.5, 0.25, 0.25, 4, 4)$ & $(0.1, 0.25, 0.25, 0.25, 0.25)$ & $(0, 0.2, 0.25, 0.25, 1, 0.25)$ \\
Giant panda vs Humpback whale & $85.17$ & $83.33$ & $83.33$ & $84.17$ & $\underline{93.75}$ & $92.36$ & $\textbf{95.83}$ \\
 & $(8, 32)$ & $(0.03125, 32)$ & $(0.25, 0.03125, 1)$ & $(0.03125, 0.1, 1)$ & $(4, 0.25, 0.25, 0.25, 2, 1)$ & $(0.1, 0.5, 0.5, 0.25, 0.5)$ & $(0, 0.4, 0.25, 0.25, 4, 4)$ \\
Giant panda vs Raccoon & $60.83$ & $67.5$ & $62.5$ & $65.17$ & $64.58$ & $\underline{70.14}$ & $\textbf{77.78}$ \\
 & $(4, 8)$ & $(0.03125, 32)$ & $(0.5, 1, 2)$ & $(8, 0.9, 4)$ & $(2, 1, 0.25, 0.25, 4, 4)$ & $(0.1, 0.25, 0.25, 0.25, 0.25)$ & $(0, 0.8, 0.25, 2, 0.5, 4)$ \\
Giant panda vs Rat & $69.33$ & $64.17$ & $68.67$ & $65.67$ & $\underline{70.14}$ & $64.58$ & $\textbf{70.42}$ \\
 & $(2, 1)$ & $(0.03125, 32)$ & $(0.0625, 4, 2)$ & $(0.0625, 0.1, 1)$ & $(0.5, 0.25, 4, 4, 1, 4)$ & $(0.1, 0.25, 0.25, 0.25, 0.25)$ & $(0.5, 1, 1, 4, 0.5, 4)$ \\
Giant panda vs Seal & $75$ & $72$ & $74.5$ & $73.33$ & $\textbf{86.81}$ & $73.61$ & $\underline{79.17}$ \\
 & $(2, 1)$ & $(0.03125, 32)$ & $(0.25, 0.25, 2)$ & $(0.03125, 0.1, 0.5)$ & $(1, 0.5, 0.25, 0.25, 2, 1)$ & $(0.1, 2, 2, 0.25, 0.25)$ & $(2, 0.8, 0.25, 0.25, 0.25, 4)$ \\
Leopard vs Persian cat & $81.67$ & $80$ & $77.5$ & $80$ & $80.56$ & $\underline{81.94}$ & $\textbf{82.25}$ \\
 & $(8, 8)$ & $(0.03125, 32)$ & $(2, 8, 2)$ & $(0.03125, 0.1, 0.5)$ & $(2, 0.25, 0.25, 0.25, 4, 1)$ & $(0.1, 2, 2, 0.25, 0.25)$ & $(0.5, 0.2, 0.25, 0.25, 0.25, 4)$ \\
Leopard vs Pig & $\underline{70}$ & $66.67$ & $63.33$ & $67.5$ & $68.75$ & $65.97$ & $\textbf{70.58}$ \\
 & $(2, 0.5)$ & $(0.03125, 32)$ & $(0.5, 0.5, 0.5)$ & $(0.03125, 0.1, 0.25)$ & $(1, 0.25, 0.25, 0.25, 2, 0.5)$ & $(0.1, 0.25, 0.25, 0.25, 0.25)$ & $(0, 0.4, 0.25, 0.25, 0.25, 4)$ \\
Leopard vs Hippopotamus & $71.67$ & $72.5$ & $70$ & $69.17$ & $\underline{75}$ & $70.83$ & $\textbf{76.39}$ \\
 & $(1, 1)$ & $(0.03125, 32)$ & $(1, 0.03125, 1)$ & $(0.03125, 0.1, 0.25)$ & $(2, 0.5, 0.25, 0.25, 0.25, 2)$ & $(0.1, 0.25, 0.25, 0.25, 0.25)$ & $(1.5, 0.6, 0.25, 0.25, 0.25, 4)$ \\
Leopard vs Humpback whale & $84.17$ & $80$ & $86.33$ & $85.83$ & $89.58$ & $\underline{92.36}$ & $\textbf{95.14}$ \\
 & $(4, 0.5)$ & $(0.03125, 32)$ & $(0.03125, 8, 2)$ & $(0.03125, 0.1, 1)$ & $(0.5, 0.5, 0.25, 0.25, 0.25, 2)$ & $(1, 0.25, 4, 0.25, 0.25)$ & $(0, 0.6, 0.25, 0.25, 2, 4)$ \\
Leopard vs Raccoon & $62.5$ & $64.17$ & $60$ & $\textbf{68.33}$ & $56.94$ & $\underline{66.67}$ & $65.28$ \\
 & $(8, 0.5)$ & $(0.03125, 32)$ & $(0.03125, 2, 1)$ & $(0.03125, 0.1, 0.125)$ & $(4, 1, 0.25, 0.25, 1, 2)$ & $(0.1, 0.25, 0.25, 0.25, 0.25)$ & $(0.5, 0.4, 0.25, 0.25, 0.5, 4)$ \\
Leopard vs Rat & $69.17$ & $\textbf{76.67}$ & $70$ & $\underline{75.83}$ & $65.28$ & $65.28$ & $70.83$ \\
 & $(1, 0.5)$ & $(0.03125, 32)$ & $(4, 0.03125, 2)$ & $(0.03125, 0.1, 1)$ & $(2, 0.5, 0.25, 0.25, 4, 2)$ & $(0.1, 0.25, 0.25, 0.25, 0.25)$ & $(1.5, 0.6, 0.25, 0.25, 0.25, 1)$ \\
Leopard vs Seal & $70.83$ & $68.33$ & $67.5$ & $70$ & $81.25$ & $\underline{82.64}$ & $\textbf{84.03}$ \\
 & $(4, 0.5)$ & $(0.03125, 32)$ & $(0.5, 1, 2)$ & $(0.03125, 0.1, 0.5)$ & $(1, 0.5, 0.25, 0.25, 4, 1)$ & $(0.1, 4, 2, 2, 4)$ & $(0.5, 1, 2, 4, 0.25, 4)$ \\
Persian cat vs Pig & $75.83$ & $\underline{77.5}$ & $73.33$ & $\underline{77.5}$ & $69.44$ & $70.14$ & $\textbf{78.36}$ \\
 & $(1, 2)$ & $(0.03125, 32)$ & $(8, 0.03125, 8)$ & $(4, 0.9, 8)$ & $(4, 2, 0.25, 0.25, 0.25, 4)$ & $(0.1, 0.25, 0.25, 0.25, 0.25)$ & $(0.5, 1, 0.25, 0.25, 1, 4)$ \\
Persian cat vs Hippopotamus & $79.33$ & $70$ & $77.5$ & $72.5$ & $75.69$ & $\underline{80.56}$ & $\textbf{81.25}$ \\
 & $(2, 1)$ & $(0.03125, 32)$ & $(2, 0.5, 2)$ & $(0.0625, 0.9, 2)$ & $(2, 0.25, 0.25, 0.25, 2, 4)$ & $(1, 0.25, 0.25, 1, 0.25)$ & $(0, 0.4, 0.25, 0.25, 0.25, 2)$ \\
Persian cat vs Humpback whale & $72.5$ & $70.83$ & $70.83$ & $74.67$ & $85.42$ & $\textbf{89.58}$ & $\underline{86.11}$ \\
 & $(1, 1)$ & $(0.03125, 32)$ & $(0.0625, 0.125, 2)$ & $(0.0625, 0.9, 16)$ & $(0.5, 0.5, 0.25, 4, 2, 4)$ & $(10, 0.25, 0.25, 4, 4)$ & $(1, 0.2, 0.25, 0.25, 0.25, 1)$ \\
Persian cat vs Raccoon & $68.33$ & $70.83$ & $68.83$ & $70.83$ & $65.97$ & $\underline{79.86}$ & $\textbf{80.56}$ \\
 & $(4, 1)$ & $(0.03125, 32)$ & $(0.03125, 0.03125, 0.03125)$ & $(16, 0.9, 2)$ & $(4, 0.25, 0.25, 0.25, 1, 1)$ & $(10, 0.25, 2, 0.5, 4)$ & $(0.5, 0.8, 2, 4, 0.25, 2)$ \\
Persian cat vs Rat & $\underline{62.50}$ & $48.33$ & $53.33$ & $49.17$ & $56.94$ & $56.94$ & $\textbf{65.28}$ \\
 & $(2, 4)$ & $(0.03125, 32)$ & $(8, 4, 8)$ & $(4, 0.9, 8)$ & $(0.5, 1, 0.25, 0.25, 0.25, 4)$ & $(0.1, 1, 4, 0.5, 4)$ & $(1.5, 1, 0.25, 0.25, 0.25, 1)$ \\
Persian cat vs Seal & $82.5$ & $70.83$ & $\textbf{85.83}$ & $75.83$ & $\underline{83.33}$ & $75$ & $76.39$ \\
 & $(1, 4)$ & $(0.03125, 32)$ & $(0.25, 0.03125, 2)$ & $(0.0625, 0.9, 4)$ & $(0.5, 0.25, 0.25, 0.25, 0.5, 2)$ & $(0.1, 0.25, 0.25, 0.25, 4)$ & $(1, 0.8, 0.25, 4, 0.25, 4)$ \\
Pig vs Hippopotamus & $64.17$ & $67.5$ & $68.33$ & $68.33$ & $\underline{72.22}$ & $68.75$ & $\textbf{74.58}$ \\
 & $(8, 16)$ & $(0.03125, 0.03125)$ & $(1, 0.03125, 8)$ & $(0.03125, 0.1, 0.25)$ & $(0.25, 0.25, 0.25, 0.25, 2, 1)$ & $(0.1, 0.25, 0.25, 0.25, 4)$ & $(1.5, 1, 0.25, 2, 4, 2)$ \\
Pig vs Humpback whale & $85$ & $82.5$ & $87.5$ & $85.83$ & $\textbf{88.89}$ & $\underline{87.50}$ & $\textbf{88.89}$ \\
 & $(8, 16)$ & $(0.03125, 32)$ & $(0.0625, 32, 2)$ & $(0.03125, 0.1, 0.5)$ & $(1, 0.25, 0.25, 0.25, 0.25, 4)$ & $(10, 0.25, 0.25, 0.25, 4)$ & $(0, 0.4, 0.25, 0.25, 0.25, 4)$ \\
Pig vs Raccoon & $60.83$ & $55.83$ & $55$ & $48.33$ & $\underline{62.50}$ & $60.42$ & $\textbf{72.22}$ \\
 & $(8, 32)$ & $(0.03125, 0.03125)$ & $(16, 1, 16)$ & $(2, 0.9, 16)$ & $(1, 0.25, 0.25, 0.25, 2, 0.5)$ & $(10, 0.25, 0.25, 0.25, 4)$ & $(1.5, 1, 0.25, 0.25, 4, 4)$ \\
Pig vs Rat & $59.17$ & $54.17$ & $54.17$ & $50.83$ & $\textbf{64.58}$ & $\underline{63.19}$ & $52.78$ \\
 & $(2, 2)$ & $(0.03125, 0.03125)$ & $(0.125, 1, 2)$ & $(0.03125, 0.1, 0.5)$ & $(4, 0.25, 0.25, 0.25, 4, 2)$ & $(0.1, 0.25, 0.25, 2, 4)$ & $(1.5, 0.8, 0.25, 0.25, 2, 0.5)$ \\
Pig vs Seal & $\underline{75}$ & $71.67$ & $69.17$ & $72.67$ & $72.92$ & $72.22$ & $\textbf{75.36}$ \\
 & $(2, 1)$ & $(0.03125, 32)$ & $(0.25, 0.03125, 2)$ & $(0.03125, 0.1, 1)$ & $(1, 0.25, 0.25, 0.25, 1, 2)$ & $(1, 0.25, 0.25, 0.5, 4)$ & $(2, 0.8, 0.25, 0.25, 1, 2)$ \\
Hippopotamus vs Humpback whale & $77.5$ & $74.17$ & $78.33$ & $75.83$ & $\underline{79.86}$ & $74.31$ & $\textbf{81.25}$ \\
 & $(1, 4)$ & $(0.03125, 32)$ & $(0.25, 8, 4)$ & $(0.03125, 0.1, 0.5)$ & $(1, 0.25, 0.25, 0.25, 0.25, 4)$ & $(1, 0.25, 2, 0.5, 4)$ & $(0.5, 1, 1, 1, 0.25, 0.25)$ \\
Hippopotamus vs Raccoon & $70.5$ & $70.67$ & $70.83$ & $69.17$ & $\textbf{75.69}$ & $75$ & $\underline{72.22}$ \\
 & $(4, 8)$ & $(0.03125, 32)$ & $(0.5, 0.03125, 1)$ & $(0.0625, 0.9, 1)$ & $(1, 0.25, 2, 4, 1, 4)$ & $(10, 0.25, 0.25, 0.25, 4)$ & $(2, 0.2, 4, 4, 0.25, 0.25)$ \\
Hippopotamus vs Rat & $68.33$ & $63.33$ & $68$ & $59.17$ & $\underline{64.58}$ & $\textbf{75.69}$ & $60.42$ \\
 & $(1, 0.5)$ & $(0.03125, 32)$ & $(0.03125, 2, 2)$ & $(0.03125, 0.1, 1)$ & $(1, 1, 0.25, 0.25, 2, 0.5)$ & $(1, 0.25, 0.25, 0.25, 1)$ & $(0, 0.6, 0.5, 0.5, 0.25, 2)$ \\
Hippopotamus vs Seal & $69.17$ & $\underline{70.33}$ & $\textbf{70.83}$ & $\textbf{70.83}$ & $60.42$ & $67.36$ & $68.06$ \\
 & $(2, 0.5)$ & $(0.03125, 16)$ & $(0.5, 0.5, 2)$ & $(0.03125, 0.1, 0.5)$ & $(2, 2, 4, 4, 0.25, 4)$ & $(0.1, 2, 4, 0.25, 4)$ & $(0, 0.6, 0.5, 0.5, 1, 4)$ \\
Humpback whale vs Raccoon & $80.67$ & $74.17$ & $80.17$ & $80$ & $83.33$ & $\textbf{91.67}$ & $\underline{88.89}$ \\
 & $(4, 2)$ & $(0.03125, 32)$ & $(2, 0.03125, 2)$ & $(4, 0.5, 4)$ & $(1, 0.5, 0.25, 0.25, 0.5, 2)$ & $(0.1, 0.25, 0.25, 1, 4)$ & $(1.5, 0.4, 0.25, 0.25, 0.25, 0.25)$ \\
Humpback whale vs Rat & $\textbf{80}$ & $78.17$ & $\textbf{80}$ & $\underline{79.17}$ & $77.78$ & $75.69$ & $\textbf{80}$ \\
 & $(2, 2)$ & $(0.03125, 32)$ & $(0.03125, 16, 8)$ & $(8, 0.9, 4)$ & $(1, 0.5, 0.25, 2, 0.25, 0.5)$ & $(1, 0.25, 0.25, 0.25, 2)$ & $(1, 0.4, 0.25, 0.25, 0.25, 0.25)$ \\
Humpback whale vs Seal & $75$ & $76.67$ & $77.5$ & $74.67$ & $\underline{78.47}$ & $74.31$ & $\textbf{79.44}$ \\
 & $(1, 1)$ & $(2, 32)$ & $(2, 0.03125, 8)$ & $(8, 0.9, 4)$ & $(2, 0.5, 0.25, 0.25, 0.5, 4)$ & $(0.1, 0.25, 0.25, 0.25, 1)$ & $(1, 0.8, 0.25, 0.25, 0.25, 0.25)$ \\
Raccoon vs Rat & $71.33$ & $\textbf{72.5}$ & $59.17$ & $\underline{71.67}$ & $65.28$ & $54.17$ & $61.81$ \\
 & $(2, 0.5)$ & $(0.03125, 32)$ & $(0.03125, 2, 2)$ & $(0.0625, 0.1, 1)$ & $(0.25, 1, 0.25, 0.25, 0.25, 4)$ & $(1, 2, 4, 0.25, 2)$ & $(1, 0.8, 0.25, 0.25, 0.25, 0.25)$ \\
Raccoon vs Seal & $72.5$ & $68.33$ & $71.67$ & $71.67$ & $75.69$ & $\underline{77.78}$ & $\textbf{80.56}$ \\
 & $(1, 2)$ & $(0.03125, 32)$ & $(0.5, 0.03125, 2)$ & $(0.03125, 0.1, 1)$ & $(1, 0.5, 0.25, 0.25, 0.25, 4)$ & $(0.1, 0.25, 0.25, 0.25, 1)$ & $(0, 0.6, 0.25, 0.25, 0.25, 1)$ \\
Rat vs Seal & $65$ & $60$ & $64.17$ & $66.67$ & $\underline{69.87}$ & $\textbf{75.69}$ & $65.28$ \\
 & $(2, 1)$ & $(0.03125, 32)$ & $(0.03125, 32, 16)$ & $(0.0625, 0.9, 2)$ & $(0.5, 0.25, 0.25, 4, 0.5, 8)$ & $(1, 0.25, 4, 0.25, 4)$ & $(1.5, 1, 2, 4, 0.25, 4)$ \\ \hline
Average ACC & $71.92$ & $69.4$ & $71.02$ & $70.5$ & $73.33$ & $\underline{74.16}$ & $\textbf{76.71}$ \\ \hline
Average Rank & $4.11$ & $5.51$ & $4.50$ & $4.78$ & $3.62$ & $\underline{3.60}$ & $\textbf{1.88}$ \\  \hline
 \multicolumn{8}{l}{$^{\dagger}$ represents the proposed models.}\\
 \multicolumn{8}{l}{The boldface and underline indicate the best and second-best models, respectively, in terms of ACC.}
\end{tabular}}
\end{table*}
\subsection{Experiments on AwA datasets}
In this subsection, we perform a detailed analysis by comparing the proposed Wave-MvSVM model with baseline models using the AwA dataset. This dataset consists of 30,475 images from 50 animal classes, with each image represented by six pre-extracted features. For our evaluation, we focus on ten test classes: chimpanzee, Persian cat, leopard, raccoon, humpback whale, giant panda, pig, hippopotamus, seal, and rat, amounting to a total of $6,180$ images. The $2000$-dimensional $L_1$ normalized Speeded-Up Robust Features (SURF) are denoted as view 1, while the $252$-dimensional Histogram of Oriented Gradient (HOG) feature descriptors are represented as view 2. We employ a one-against-one strategy for each combination of class pairs to train $45$ binary classifiers. We evaluate the performance of the proposed Wave-MvSVM model along with the baseline models using ACC metrics and corresponding optimal hyperparameters, which are reported in Table \ref{Classification performance of AwA in nonLinear Case.}. The average ACC of the proposed Wave-MvSVM and the baseline SVM2K, MvTSVM, PSVM-2V, MVNPSVM, MVLDM, and MVCSKL models are $76.71\%$, $71.92\%$, $69.40\%$, $71.02\%$, $70.50\%$, $73.33\%$, and $74.16\%$, respectively. Table \ref{Classification performance of AwA in nonLinear Case.} presents the average rank of the proposed Wave-MvSVM and the baseline models. The proposed Wave-MvSVM model achieves the highest average ACC, the lowest average rank, and the most wins, demonstrating their superior performance. Now, we conduct the Friedman statistical test, followed by the Nemenyi post hoc tests.  For the significance level of $\alpha = 0.05$, we calculated as $\chi_F^2 = 76.65$ and $F_F = 17.44$. The tabulated value is obtained by $F_F(6, 264) = 2.133$. The null hypothesis is rejected as $17.44 > F_F(6, 264)$. Now, the Nemenyi post-hoc test is employed to identify significant differences among the pairwise comparisons of the models. We calculate $C.D. = 1.34$, which indicates that the average rankings of the models in Table \ref{Classification performance of AwA in nonLinear Case.} should have a minimum difference of $1.34$ to be considered statistically significant. The differences in average ranks between the proposed Wave-MvSVM and the baseline SVM2K, MvTSVM, PSVM-2V, MVNPSVM, MVLDM, and MVCSKL models are $2.23$, $3.63$, $2.62$, $2.90$, $1.74$, and $1.72$. The observed differences are greater than the critical difference \( C.D. \). Consequently, based on the Nemenyi post hoc test, significant distinctions are found between the proposed Wave-MvSVM model and the baseline models SVM2K, MvTSVM, PSVM-2V, MVNPSVM, MVLDM, and MVCSKL. Consequently, the proposed Wave-MvSVM model demonstrates superior performance compared to baseline models. Considering the above discussion based on ACC, rank, and statistical tests, we can conclude that the proposed Wave-MvSVM model demonstrates superior and robust performance compared to the baseline models.
\begin{table*}[ht!]
\centering
    \caption{Classification performance of proposed Wave-MvSVM model along with the baseline models on UCI and KEEL datasets with label noise.}
    \label{Classification performance UCi and KEEL with noie.}
    \resizebox{0.9\textwidth}{!}{
\begin{tabular}{lcccccccc}
\hline
Dataset & Noise & SVM2K \cite{farquhar2005two} & MvTSVM \cite{xie2015multi} & PSVM-2V \cite{tang2017multiview} & MVNPSVM \cite{tang2018multi} & MVLDM \cite{hu2024multiview} & MVCSKL \cite{tang2023multi} & Wave-MvSVM$^{\dagger}$ \\
&  & $(C_1, \sigma)$ & $(C_1, \sigma)$ & $(C_1, \gamma, \sigma)$ & $(C_1, D, \sigma)$ & $(C_1, v_1, v_2, \theta, \sigma)$ & $(a_1, d, \gamma, c_1, \sigma)$ & $(a_1, \lambda, \mathcal{C}_1, _2, D, \sigma)$ \\ \hline
blood & $5\%$ & $75$ & $\textbf{77.68}$ & $\underline{77.23}$ & $73.49$ & $\textbf{77.68}$ & $72.06$ & $\textbf{77.68}$ \\
 &  & $(2, 0.5)$ & $(0.03125, 32)$ & $(0.03125, 32, 16)$ & $(4, 0.9, 8)$ & $(0.25, 0.25, 0.25, 0.25, 0.5, 1)$ & $(0.1, 1, 1, 0.25, 0.5)$ & $(0, 0.6, 0.25, 0.25, 1, 2)$ \\
 & $10\%$ & $\underline{79.46}$ & $77.23$ & $77.23$ & $61.91$ & $77.23$ & $77.68$ & $\textbf{79.68}$ \\
 &  & $(2, 0.5)$ & $(8, 32)$ & $(0.5, 0.25, 1)$ & $(0.0625, 0.1, 0.5)$ & $(0.25, 0.25, 4, 4, 0.5, 2)$ & $(0.1, 0.25, 0.25, 0.25, 0.25)$ & $(2, 0.4, 0.25, 0.25, 0.25, 4)$ \\
 & $15\%$ & $\underline{77.23}$ & $\underline{77.23}$ & $\underline{77.23}$ & $74.1$ & $\underline{77.23}$ & $76.34$ & $\textbf{78.87}$ \\
 &  & $(2, 0.5)$ & $(0.03125, 32)$ & $(0.5, 0.5, 0.5)$ & $(0.03125, 0.1, 0.25)$ & $(1, 0.25, 0.25, 0.25, 2, 0.5)$ & $(0.1, 0.25, 0.25, 0.25, 0.25)$ & $(0, 0.4, 0.25, 0.25, 0.25, 4)$ \\
 & $20\%$ & $76.79$ & $77.68$ & $75.89$ & $66.17$ & $\underline{78.57}$ & $72.77$ & $\textbf{81.43}$ \\
 &  & $(2, 4)$ & $(0.03125, 32)$ & $(8, 4, 8)$ & $(4, 0.9, 8)$ & $(0.5, 1, 0.25, 0.25, 0.25, 4)$ & $(0.1, 1, 4, 0.5, 4)$ & $(1.5, 1, 0.25, 0.25, 0.25, 1)$ \\ \hline
Average ACC &  & $77.12$ & $77.46$ & $76.9$ & $68.92$ & $\underline{77.68}$ & $74.71$ & $\textbf{79.42}$ \\ \hline
breast\_cancer\_wisc\_prog & $5\%$ & $71.36$ & $76.27$ & $76.27$ & $60.66$ & $74.58$ & $\underline{77.97}$ & $\textbf{78.02}$ \\
 &  & $(8, 32)$ & $(0.03125, 32)$ & $(0.25, 0.03125, 1)$ & $(0.03125, 0.1, 1)$ & $(4, 0.25, 0.25, 0.25, 2, 1)$ & $(0.1, 0.5, 0.5, 0.25, 0.5)$ & $(0, 0.4, 0.25, 0.25, 4, 4)$ \\
 & $10\%$ & $71.19$ & $\textbf{77.97}$ & $\underline{79.66}$ & $73.31$ & $76.27$ & $74.41$ & $79.41$ \\
 &  & $(1, 1)$ & $(0.03125, 32)$ & $(1, 0.03125, 1)$ & $(0.03125, 0.1, 0.25)$ & $(2, 0.5, 0.25, 0.25, 0.25, 2)$ & $(0.1, 0.25, 0.25, 0.25, 0.25)$ & $(1.5, 0.6, 0.25, 0.25, 0.25, 4)$ \\
 & $15\%$ & $76.27$ & $\textbf{81.36}$ & $72.94$ & $74.45$ & $\underline{80.75}$ & $71.19$ & $80.71$ \\
 &  & $(4, 0.5)$ & $(0.03125, 32)$ & $(0.5, 1, 2)$ & $(0.03125, 0.1, 0.5)$ & $(1, 0.5, 0.25, 0.25, 4, 1)$ & $(0.1, 4, 2, 2, 4)$ & $(0.5, 1, 2, 4, 0.25, 4)$ \\
 & $20\%$ & $74.58$ & $72.88$ & $74.58$ & $60.69$ & $\underline{74.75}$ & $73.53$ & $\textbf{76.67}$ \\
 &  & $(2, 1)$ & $(0.03125, 32)$ & $(0.03125, 32, 16)$ & $(0.0625, 0.9, 2)$ & $(0.5, 0.25, 0.25, 4, 0.5, 8)$ & $(1, 0.25, 4, 0.25, 4)$ & $(1.5, 1, 2, 4, 0.25, 4)$ \\ \hline
Average ACC &  & $73.35$ & $\underline{77.12}$ & $75.86$ & $67.28$ & $76.59$ & $74.28$ & $\textbf{78.7}$ \\ \hline
brwisconsin & $5\%$ & $80.69$ & $81.24$ & $\underline{82.65}$ & $81.19$ & $82.31$ & $79.78$ & $\textbf{83.82}$ \\
 &  & $(1, 0.5)$ & $(2, 16)$ & $(0.125, 0.5, 1)$ & $(0.25, 0.1, 32)$ & $(1, 1, 0.25, 0.25, 1)$ & $(0.1, 0.25, 0.25, 0.25, 0.25)$ & $(0, 0.2, 0.25, 0.25, 0.25, 0.25)$ \\
 & $10\%$ & $75.98$ & $71.76$ & $73.14$ & $\textbf{79.92}$ & $\underline{79.9}$ & $74.31$ & $74.8$ \\
 &  & $(0.25, 32)$ & $(0.03125, 4)$ & $(4, 0.5, 32)$ & $(0.03125, 0.6, 32)$ & $(0.5, 0.5, 0.25, 0.25, 0.5)$ & $(1, 0.5, 0.5, 0.25, 0.5)$ & $(0, 0.2, 0.25, 0.25, 0.25, 2)$ \\
 & $15\%$ & $\textbf{76.96}$ & $73.24$ & $72.88$ & $71.19$ & $76.47$ & $76.27$ & $\underline{76.82}$ \\
 &  & $(0.5, 32)$ & $(16, 32)$ & $(0.5, 0.03125, 32)$ & $(0.03125, 0.1, 0.03125)$ & $(0.25, 0.25, 0.25, 0.25, 0.25)$ & $(0.1, 2, 2, 0.25, 0.25)$ & $(0, 0.2, 0.5, 4, 0.25, 0.25)$ \\
 & $20\%$ & $75$ & $\underline{79.22}$ & $76.76$ & $72$ & $75.29$ & $71.37$ & $\textbf{79.9}$ \\
 &  & $(0.03125, 0.03125)$ & $(0.03125, 0.03125)$ & $(0.03125, 0.03125, 1)$ & $(2, 0.1, 32)$ & $(0.25, 0.25, 0.25, 0.25, 0.25)$ & $(1, 0.25, 0.25, 0.25, 0.25)$ & $(0.5, 1, 1, 0.25, 0.25, 1)$ \\ \hline
Average ACC &  & $77.16$ & $76.37$ & $76.36$ & $76.08$ & $\underline{78.49}$ & $75.43$ & $\textbf{78.84}$ \\ \hline
congressional\_voting & $5\%$ & $\underline{63.08}$ & $62.31$ & $62.31$ & $60.66$ & $\textbf{65.81}$ & $62.95$ & $60.42$ \\
 &  & $(1, 16)$ & $(0.03125, 16)$ & $(1, 0.125, 32)$ & $(0.03125, 0.3, 32)$ & $(0.25, 0.25, 0.25, 0.25, 0.25)$ & $(0.1, 2, 2, 0.25, 2)$ & $(1.5, 0.8, 0.25, 0.25, 0.25, 0.25)$ \\
 & $10\%$ & $\textbf{60.77}$ & $\textbf{60.77}$ & $\textbf{60.77}$ & $57.61$ & $\underline{57.69}$ & $56.92$ & $54.21$ \\
 &  & $(16, 32)$ & $(0.03125, 2)$ & $(0.03125, 2, 2)$ & $(0.03125, 0.1, 0.03125)$ & $(0.25, 0.25, 0.25, 0.25, 0.25)$ & $(10, 0.5, 0.5, 0.25, 2)$ & $(1.5, 1, 2, 4, 4, 0.25)$ \\
 & $15\%$ & $62.31$ & $64.62$ & $\textbf{69.09}$ & $64.27$ & $63.08$ & $\underline{68.46}$ & $67.97$ \\
 &  & $(2, 1)$ & $(0.03125, 2)$ & $(4, 0.03125, 4)$ & $(0.03125, 0.1, 0.03125)$ & $(0.25, 0.25, 0.25, 0.25, 0.25)$ & $(0.1, 1, 1, 0.25, 0.25)$ & $(1, 0.2, 0.5, 0.25, 0.25, 4)$ \\
 & $20\%$ & $61.54$ & $\underline{63.08}$ & $\textbf{64.62}$ & $59.08$ & $58.46$ & $59.42$ & $59.37$ \\
 &  & $(2, 0.25)$ & $(0.03125, 32)$ & $(8, 0.25, 4)$ & $(1, 0.9, 8)$ & $(4, 4, 0.25, 4, 4)$ & $(0.1, 0.5, 0.5, 0.25, 0.25)$ & $(2, 1, 0.25, 0.25, 0.25, 0.5)$ \\ \hline
Average ACC &  & $61.93$ & $\underline{62.7}$ & $\textbf{64.2}$ & $60.41$ & $61.26$ & $61.94$ & $60.49$ \\ \hline
credit\_approval & $5\%$ & $75.36$ & $74.11$ & $76.96$ & $\underline{77.7}$ & $74.06$ & $76.21$ & $\textbf{78.25}$ \\
 &  & $(4, 32)$ & $(0.03125, 16)$ & $(4, 0.03125, 16)$ & $(0.25, 0.9, 32)$ & $(0.25, 0.25, 0.25, 0.25, 0.25)$ & $(0.1, 0.25, 0.25, 0.25, 4)$ & $(0, 0.8, 0.5, 0.25, 0.25, 4)$ \\
 & $10\%$ & $70.05$ & $74.4$ & $69.37$ & $71.31$ & $\textbf{75.51}$ & $70.97$ & $\underline{75.22}$ \\
 &  & $(16, 1)$ & $(0.03125, 0.25)$ & $(1, 0.03125, 1)$ & $(16, 0.9, 16)$ & $(2, 2, 0.25, 0.25, 2)$ & $(1, 0.25, 0.25, 0.25, 0.25)$ & $(0, 0.2, 1, 0.25, 0.25, 1)$ \\
 & $15\%$ & $61.84$ & $64.4$ & $66.3$ & $63.61$ & $\underline{68.74}$ & $65.81$ & $\textbf{69.25}$ \\
 &  & $(16, 0.125)$ & $(0.5, 32)$ & $(0.03125, 0.03125, 0.03125)$ & $(1, 0.4, 8)$ & $(2, 0.25, 0.25, 4, 0.25)$ & $(0.1, 2, 0.25, 0.25, 4)$ & $(0.5, 0.2, 4, 0.25, 0.25, 4)$ \\
 & $20\%$ & $68.6$ & $69.28$ & $62.61$ & $68.51$ & $\underline{76.33}$ & $73.08$ & $\textbf{77.19}$ \\
 &  & $(4, 4)$ & $(0.03125, 32)$ & $(0.5, 0.03125, 4)$ & $(0.5, 0.9, 32)$ & $(0.25, 0.25, 0.25, 0.25, 0.25)$ & $(0.1, 2, 2, 0.25, 0.25)$ & $(2, 0.8, 0.5, 0.25, 0.25, 4)$ \\ \hline
Average ACC&  & $68.96$ & $70.55$ & $68.81$ & $70.28$ & $\underline{73.66}$ & $71.52$ & $\textbf{74.98}$ \\ \hline
Overall average ACC &  & $71.70$ & $72.84$ & $72.42$ & $68.59$ & $\underline{73.54}$ & $71.58$ & $\textbf{74.49}$ \\ \hline
 \multicolumn{8}{l}{$^{\dagger}$ represents the proposed models.}\\
 \multicolumn{8}{l}{The boldface and underline indicate the best and second-best models, respectively, in terms of ACC.}
\end{tabular}}
\end{table*}
\subsection{Evaluation on UCI and KEEL datasets with added label noise}
The UCI and KEEL datasets used in our evaluation reflect real-world scenarios. However, it is crucial to acknowledge that data impurities or noise can arise from various factors. In these situations, it is vital to create a robust model capable of managing such challenges effectively. To demonstrate the efficacy of the proposed Wave-MvSVM model under adverse conditions, we intentionally introduced label noise into selected datasets. For our comparative analysis, we chose five diverse datasets: blood, breast\_cancer\_wisc\_prog, brwisconsin, congressional\_voting, and credit\_approval. We selected one dataset with the proposed Wave-MvSVM model to ensure impartiality in evaluating the models. To carry out a thorough analysis, we introduced label noise at varying levels of $5\%$, $10\%$, $15\%$, and $20\%$ to corrupt the labels. The average accuracies of all the models for the selected datasets with $5\%$, $10\%$, $15\%$, and $20\%$ noise levels are presented in Table \ref{Classification performance UCi and KEEL with noie.}. The average ACC of the proposed Wave-MvSVM model on blood with various noise levels is $79.42\%$, surpassing the performance of the baseline models. The ACC at $20\%$ level of noise is $81.43\%$, which surpasses at $0\%$ level of noise. On the breast\_cancer\_wisc\_prog dataset, the average ACC of the proposed Wave-MvRVFl model is $78.70\%$, which secured the top position comparison to the baseline models. However, on the congressional\_voting dataset, the proposed models did not achieve the top performance compared to the baseline models. Nonetheless, they did secure the top position at the $0\%$ level of noise with an ACC of $66.23\%$. On brwisconsin, and credit\_approval datasets, the proposed models show the best performance with an average ACC of $78.84\%$, and $74.98\%$, respectively. On various levels of noise, the proposed Wave-MvSVM achieved an overall average ACC of $74.49\%$, surpassing all the baseline models. These results highlight the significance of the proposed Wave-MvSVM model as a robust solution capable of performing effectively in challenging environments marked by noise and impurities.
\subsection{Sensitivity analysis}
In this section, we perform the sensitivity analysis of several key hyperparameters of the proposed Wave-MvSVM model. These analyses covered various factors, including hyperparameters corresponding to the W-loss function $\lambda$ and $a$, discussed in subsection \ref{Sensitivity Analysis of wave loss parameters}. We examined the effects of different levels of label noise in subsection \ref{Sensitivity analysis of label noise}. Finally, we assess the influence of hyperparameters $\mathcal{C}_1$ and $\mathcal{C}_2$, discussed in subsection \ref{Sensitivity analysis of hyperparameter Wave-MvSVM}.
\begin{figure*}[ht!]
\begin{minipage}{.5\linewidth}
\centering
\subfloat[Chimpanzee vs \\ Giant panda]{\label{fig1} \includegraphics[scale=0.4]{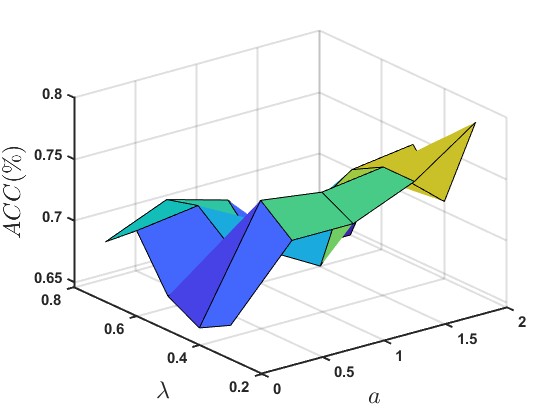}}
\end{minipage}
\begin{minipage}{.5\linewidth}
\centering
\subfloat[Hippopotamus vs \\ Humpback whale]{\label{fig2} \includegraphics[scale=0.4]{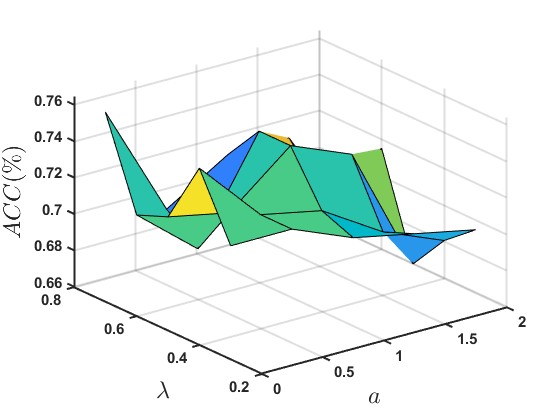}}
\end{minipage}
\begin{minipage}{.5\linewidth}
\centering
\subfloat[Chimpanzee vs \\ Humpback whale]{\label{fig3} \includegraphics[scale=0.4]{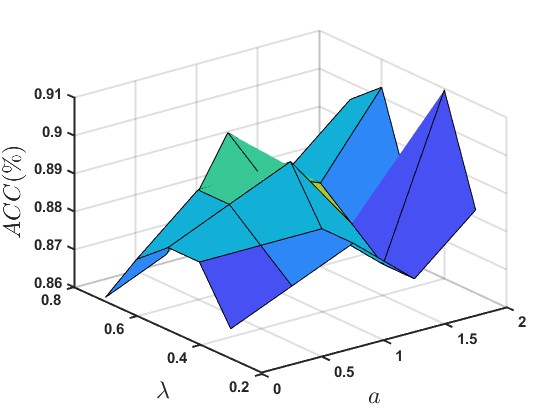}}
\end{minipage}
\begin{minipage}{.5\linewidth}
\centering
\subfloat[Humpback whale vs \\ Raccoon]{\label{fig4} \includegraphics[scale=0.34]{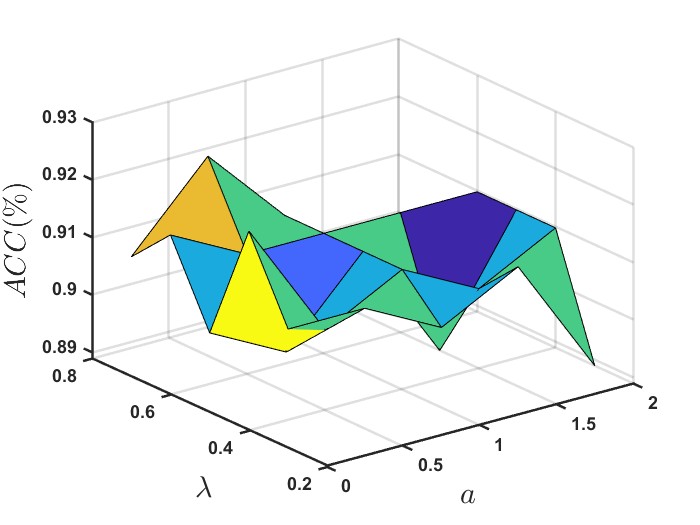}}
\end{minipage}
\caption{Effect of parameters of W-loss function $\lambda$ and $a$ on the performance of the proposed Wave-MvSVM model on AwA datasets.}
\label{effect of parameter lambda and a}
\end{figure*}
\begin{figure*}[ht!]
\begin{minipage}{.5\linewidth}
\centering
\subfloat[blood]{\includegraphics[scale=0.4]{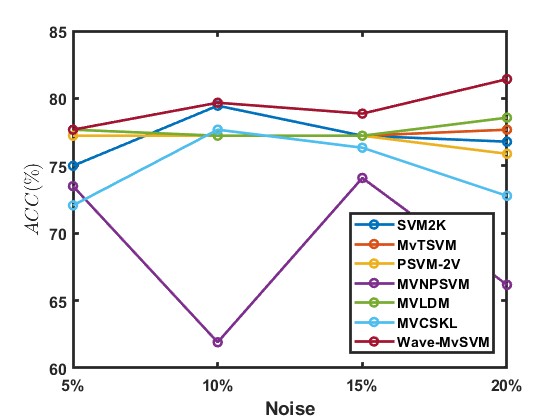}}
\end{minipage}
\begin{minipage}{.5\linewidth}
\centering
\subfloat[breast\_cancer\_wisc\_prog]{\includegraphics[scale=0.4]{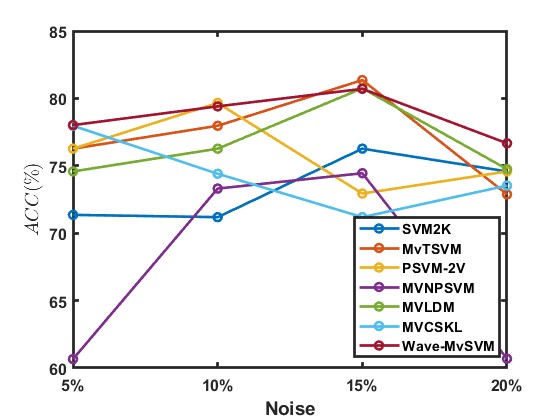}}
\end{minipage}
\par\medskip
\begin{minipage}{.5\linewidth}
\centering
\subfloat[congressional\_voting]{\includegraphics[scale=0.4]{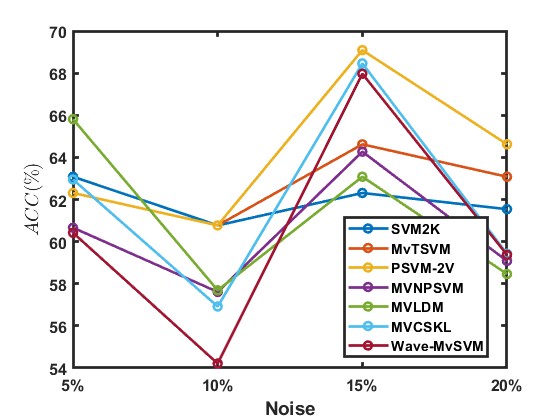}}
\end{minipage}
\begin{minipage}{.5\linewidth}
\centering
\subfloat[credit\_approval]{\includegraphics[scale=0.4]{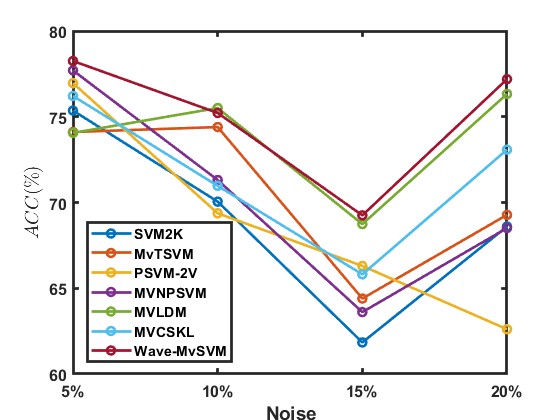}}
\end{minipage}
\caption{Effect of different levels of label noise on the performance of the proposed Wave-MvRVFL model on UCI and KEEL datasets.}
\label{labels of noise Curve}
\end{figure*}
\begin{figure*}[ht!]
\begin{minipage}{.5\linewidth}
\centering
\subfloat[cleve]{\includegraphics[scale=0.4]{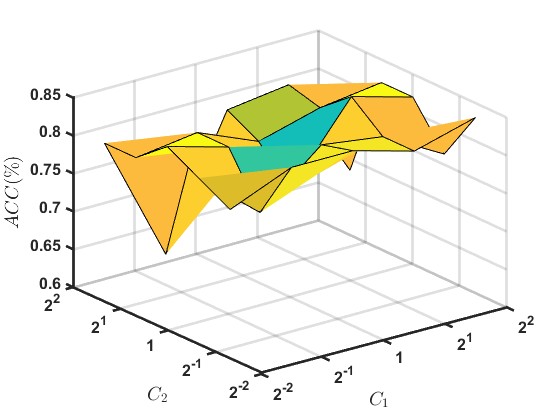}}
\end{minipage}
\begin{minipage}{.5\linewidth}
\centering
\subfloat[hepatitis]{\includegraphics[scale=0.4]{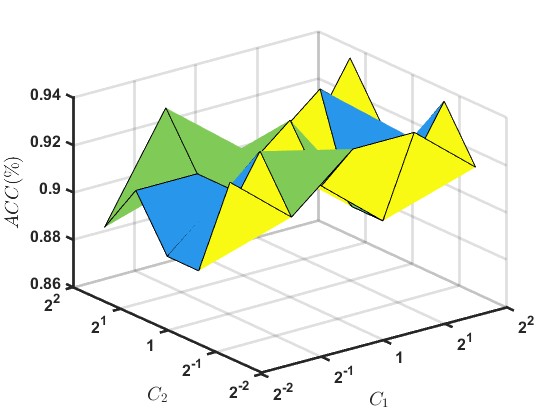}}
\end{minipage}
\par\medskip
\begin{minipage}{.5\linewidth}
\centering
\subfloat[monks\_1]{\includegraphics[scale=0.4]{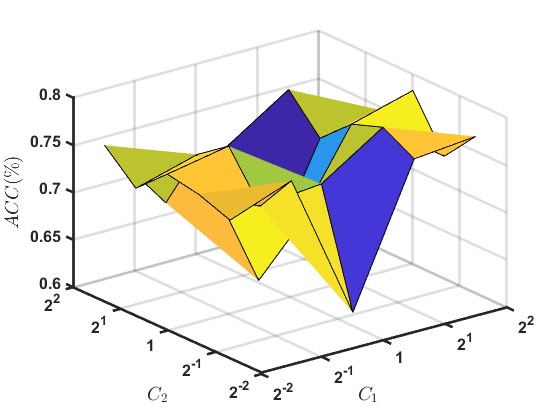}}
\end{minipage}
\begin{minipage}{.5\linewidth}
\centering
\subfloat[statlog\_heart]{\includegraphics[scale=0.4]{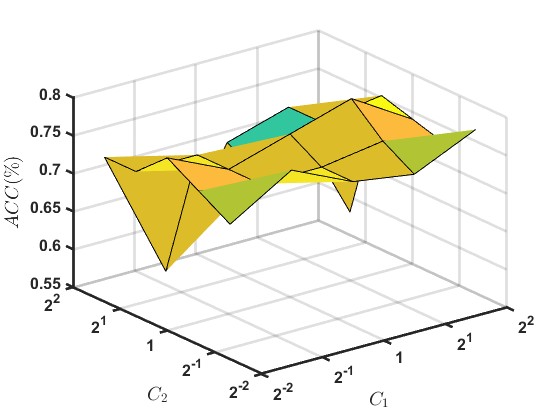}}
\end{minipage}
\caption{Effect of parameters $\mathcal{C}_1$ and $\mathcal{C}_2$ on the performance of the proposed Wave-MvSVM model on UCI and KEEL datasets.}
\label{effect of parameter c1 and c2}
\end{figure*}
\subsubsection{Sensitivity analysis of W-loss parameters \texorpdfstring{$\lambda$}{lambda} and \texorpdfstring{$a$}{a}}
\label{Sensitivity Analysis of wave loss parameters}
To thoroughly understand the robustness of the W-loss function, it is essential to analyze its sensitivity to the hyperparameters $\lambda$ and $a$. This thorough exploration enables us to identify the configuration that maximizes predictive ACC and enhances the model's resilience when confronted with unseen data. Figure \ref{effect of parameter lambda and a} illustrates a significant fluctuation in the model's ACC across various $\lambda$ and $a$ values, underscoring the sensitivity of our model's performance to these specific hyperparameters. From Figures \ref{fig1} and \ref{fig3}, it is observed that the optimal performance of the proposed Wave-MvSVM model is within the $\lambda$ and $a$ ranges of $0.2$ to $0.6$ and $1$ to $2$, respectively. Similarly, from Figures \ref{fig2} and \ref{fig4}, the ACC of the proposed Wave-MvSVM model achieves the maximum when $\lambda$ and $a$ range from $0.4$ to $0.8$, and $0$ to $1.5$ respectively. Therefore, we recommend utilizing $\lambda$ and $a$ from the range $0.2$ to $0.6$, and $0.5$ to $1.5$ for optimal results, although fine-tuning may be necessary depending on the dataset's characteristics to achieve optimal generalization performance for the proposed Wave-MvSVM model.
\subsubsection{Sensitivity analysis of label noise}
\label{Sensitivity analysis of label noise}
One of the main focus of the proposed Wave-MvSVM model is to mitigate the adverse impact of noise. To thoroughly evaluate the robustness of the Wave-MvSVM model, we introduced label noise into four datasets: blood, breast\_cancer\_wisc\_prog, planning, and monks\_1. We varied the noise levels at $5\%$, $10\%$, $15\%$, and $20\%$. Figure \ref{labels of noise Curve} provides a detailed analysis of the results. It clearly shows that the performance of baseline models fluctuates significantly and tends to decline as the level of label noise increases. This instability indicates a susceptibility to noise, which can severely affect their classification ACC. In contrast, the Wave-MvSVM model consistently maintains superior performance across all levels of introduced noise. This stability demonstrates the model's robustness and effectiveness in handling noisy data. The Wave-MvSVM's ability to mitigate the effects of noise ensures reliable performance and ACC, even in challenging environments where data impurities are prevalent.
\subsubsection{Sensitivity analysis of hyperparameter \texorpdfstring{$C_1$}{C1} and \texorpdfstring{$\mathcal{C}_2$}{C2}}
\label{Sensitivity analysis of hyperparameter Wave-MvSVM}
We delve into the impact of hyperparameters $\mathcal{C}_1$ and $\mathcal{C}_2$ on the predictive capability of the proposed Wave-MvSVM model. The sensitivity analysis, depicted in Figure \ref{effect of parameter c1 and c2}, is conducted on both KEEL and UCI datasets, evaluating ACC variations with varying values of $\mathcal{C}_1$ and $\mathcal{C}_2$. However, beyond a certain threshold, further increments in $\mathcal{C}_1$ yield diminishing returns in testing ACC. Specifically, once the values surpass $2^{1}$, the ACC reaches a plateau, indicating that additional increases in $\mathcal{C}_1$ do not significantly enhance performance. This underscores the importance of meticulous hyperparameter selection for the proposed Wave-MvSVM model to achieve optimal generalization performance. By carefully tuning $\mathcal{C}_1$ and $\mathcal{C}_2$, practitioners can ensure the models' effectiveness in handling diverse datasets and achieving robust predictive capabilities.
\section{Conclusions and Future Work}
\label{Conclusion and Future Work}
In this paper, we propose a novel multiview support vector machine framework leveraging the wave loss function (Wave-MvSVM). It addresses the complexities of multi-view representations and handles noisy samples simultaneously within a unified framework. In line with the consensus and complementarity principles, Wave-MvSVM integrates a consensus regularization term and a combination weight strategy to enhance the utilization of multi-view representations effectively. The wave loss (W-loss) function, known for its smoothness, asymmetry, and bounded properties, proves highly effective in mitigating the detrimental impacts of noisy and outlier data, thereby improving model stability. The theoretical foundation supported by Rademacher's complexity underscores its strong generalization capabilities of the proposed Wave-MvSVM model. We utilize the ADMM and GD algorithms to solve the optimization problem of the Wave-MvSVM model. To showcase the effectiveness, robustness, and efficiency of the proposed Wave-MvSVM model, we conducted a series of rigorous experiments and subjected them to thorough statistical analyses. We conducted experiments using datasets from UCI, KEEL, and AwA. The experimental findings, supported by statistical analyses, demonstrate that the proposed Wave-MvSVM models outperform baseline models in terms of generalization performance even when tested with datasets containing introduced label noise. In the future, one can extend the proposed Wave-MvSVM to handle scenarios with multiple views (more than two views) and develop specific acceleration strategies tailored for large-scale datasets. Additionally, exploring the integration of the asymmetric and bounded loss function into deep learning frameworks represents a meaningful and intriguing avenue for future research.
\section*{Acknowledgment}
This study receives support from the Science and Engineering Research Board (SERB) through the Mathematical Research Impact-Centric Support (MATRICS) scheme Grant No. MTR/2021/000787. The Council of Scientific and Industrial Research (CSIR), New Delhi, provided a fellowship for Mushir Akhtar’s research under grant no. 09/1022(13849)/2022-EMR-I.
\bibliography{refs.bib}
\bibliographystyle{plainnat}
\end{document}